\documentclass[pageno]{jpaper}

\usepackage[normalem]{ulem}
\usepackage{xspace}
\usepackage{amsmath}
\usepackage{amsfonts}
\usepackage{amssymb}
\usepackage{mathtools}
\usepackage{amsthm}
\usepackage[ruled,vlined,linesnumbered]{algorithm2e}
\usepackage{multirow}
\usepackage{multicol}
\usepackage{makecell}
\usepackage[compress]{cite}
\newcommand{\sys}{HummingBird\xspace}

\newcommand{\Mod}[1]{\ (\mathrm{mod}\ #1)}
\newcommand{\arith}[2]{\langle#1\rangle^{#2}}
\newcommand{\bin}[1]{\langle#1\rangle^{B}}
\DeclareMathOperator*{\drelu}{DReLU}
\DeclareMathOperator*{\relu}{ReLU}
\DeclareMathOperator*{\compress}{Compress}

\theoremstyle{plain}
\newtheorem{theorem}{Theorem}

\theoremstyle{definition}

\theoremstyle{remark}

\author{
\normalsize 
Kiwan Maeng$^\dagger$, 
G. Edward Suh$^\ddagger{^\S}$, 
\\
\normalsize
$^\dagger$Pennsylvania State University,
$^\ddagger$Meta AI,
$^\S$Cornell University,
}

\begin{document}

\title{
Approximating ReLU on a Reduced Ring for Efficient MPC-based Private Inference
}

\date{}
\maketitle

\thispagestyle{empty}

\begin{abstract}
Secure multi-party computation (MPC) allows users to offload machine learning inference on untrusted servers without having to share their privacy-sensitive data.
Despite their strong security properties, MPC-based private inference has not been widely adopted in the real world due to their high communication overhead. When evaluating ReLU layers, MPC protocols incur a significant amount of communication between the parties, making the end-to-end execution time multiple orders slower than its non-private counterpart.

This paper presents \sys, an MPC framework that reduces the ReLU communication overhead significantly by using only a subset of the bits to evaluate ReLU on a smaller ring. Based on theoretical analyses, \sys identifies bits in the secret share that are not crucial for accuracy and excludes them during ReLU evaluation to reduce communication.
With its efficient search engine, \sys discards 87--91\% of the bits during ReLU and still maintains high accuracy. On a real MPC setup involving multiple servers, \sys achieves on average 2.03--2.67$\times$ end-to-end speedup without introducing any errors, and up to 8.64$\times$ average speedup when some amount of accuracy degradation can be tolerated, due to its up to 8.76$\times$ communication reduction.
\end{abstract}

\section{Introduction}

Machine learning (ML) inference often uses privacy-sensitive user data as an input feature.
A model that predicts patients' disease by looking at their X-ray images~\cite{xray} uses the patients' private X-ray data.
Code auto-completion services like GitHub CoPilot~\cite{copilot} take in the user's proprietary code snippet to fill in the rest of the code.
%
%
Smart home devices that take in the user's verbal command~\cite{alexa, googlehome, fbportal} collect the user's raw microphone inputs that can contain  sensitive information.
As ML models powering these services become larger and are often proprietary,
an increasing trend is to host these models on a remote server owned by the service provider, to which the users send their input data.
%
This emerging trend creates a dilemma for the users --- to use high-quality services empowered by large ML models, the users have to send their privacy-sensitive input data to a third party, risking potential privacy leakage.

Secure \emph{multi-party computation} (MPC;~\cite{mpc}) is gaining wide interest as a potential solution to this dilemma. MPC allows users to offload ML inference to third-party servers, without having to reveal their private data to the servers~\cite{minionn, delphi, gazelle, aby, aby3, crypten, cryptflow, cryptflow2, cheetah, deepreduce, snl}. In MPC, instead of sending their raw data, users send \emph{secret shares} of their data, from which the servers cannot infer the users' raw data. Without learning anything about the users' data, the servers run inference using the secret shares and send the result back to the users. Only the users, once they receive all the results from the servers, can retrieve the output of the inference. Figure~\ref{fig:mpc_intro} summarizes the high-level operation of an MPC-based private inference.

\begin{figure}
    \centering
    \includegraphics[width=0.49\textwidth]{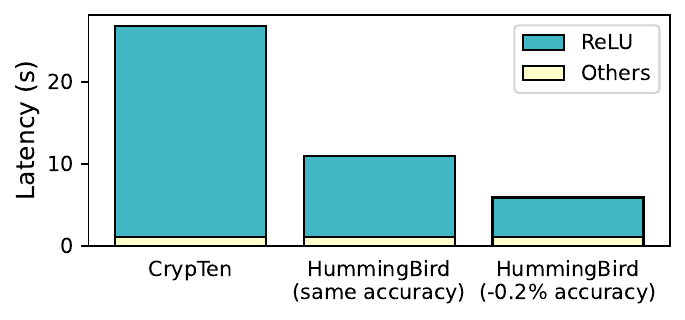}
    \caption{Latency of running \emph{512} CIFAR10 inferences on ResNet18 with CrypTen~\cite{crypten} and our proposed framework, \sys. When 0.2\% accuracy degradation is tolerated, \sys achieves a throughput of 87 samples/s (4.41$\times$ over CrypTen). Details of the setup can be found in Section~\ref{sec:eval}.}
    \label{fig:teaser}
\end{figure}

Despite their strong security guarantees, MPC-based private inference has not been widely adopted in the real world yet, due to their high runtime overheads. Even the most efficient MPC schemes~\cite{crypten, cheetah} experience multiple orders of magnitude slowdown over a non-private baseline.
%
Unlike non-private inference that are usually computation- or memory-bound, the majority of the overhead in MPC comes from communications between parties during non-linear operations --- or most prominently, ReLU. In a particular setup we studied, ReLU was accountable for over 93\% of the total overhead (Figure~\ref{fig:teaser}, leftmost bar), which is in line with observations from prior works~\cite{deepreduce}.
To tackle this unique source of overhead, recent works concentrated on designing a faster algorithm for ReLU~\cite{aby, aby3, cryptflow2, securenn, falcon} or model architectures that use less number of ReLUs~\cite{deepreduce, sphynx, snl, deepreshape, senet}.

In this paper, we explore an orthogonal approach that accelerates \emph{existing} ReLU algorithms further by approximating the sign estimation process (\emph{i.e.}, \emph{DReLU}).
%
%
The key insight is that simply \emph{guessing the sign}, unlike high-precision arithmetic operations, can still be done correctly by only looking at a small subset of bits on a smaller ring.
We theoretically show that for a large family of ReLU protocols, discarding a carefully-selected amount of high-order and low-order bits of a secret share renders the final ReLU outcome equivalent to magnitude-based activation pruning, which is empirically known to have little effect on accuracy~\cite{fatrelu, activation_sparsity, activation_sparsity2, attention_sparsity, attention_sparsity2} if done properly (Section~\ref{sec:proof}).

Based on the theoretical insight, we propose \sys, a framework that automatically selects a proper number of bits to discard for each ReLU layer and uses an optimized kernel to translate the reduced bits into an end-to-end speedup. \sys achieves 2.49--5.34$\times$ end-to-end speedup on a typical LAN setup (Figure~\ref{fig:teaser}), and up to 8.64$\times$ speedup on a network-constrained WAN setup over the popular CrypTen framework~\cite{crypten}.
\sys is orthogonal to works that reduce the number of ReLUs~\cite{deepreduce, sphynx, snl, deepreshape, senet} and can be used in synergy to further accelerate them.
Below summarizes our contributions:

%

\begin{enumerate}
    \item We theoretically show that only using a small subset of the bits of the secret shares is sufficient to keep the ReLU results close to the original result for a large family of MPC protocols. Specifically, we show that removing the majority of the high-order and low-order bits in the secret shares renders the result identical to activation pruning. The theoretical result serves as a stepping stone for \sys. 
    \item We propose an efficient search algorithm to decide how many high- and low-order bits to remove for each layer, and present an efficient search engine that performs the search on a lightweight simulation environment. Within a reasonable amount of time (several minutes to an hour), \sys finds a configuration that minimally impacts the model accuracy while significantly improving the communication overhead.
    \item We implemented a runtime library as an extension to CrypTen~\cite{crypten} that can bring up to 8.64$\times$ average end-to-end speedup and 8.76$\times$ communication reduction with the configuration found by the search algorithm. We will open-source the entire codebase, including the search engine and the runtime library, upon paper publication.
\end{enumerate}
\section{Background and Motivation}

\subsection{Private Inference with Multi-party Computation}
\label{sec:bg_mpc}

With the rising concerns on data privacy in ML-based services, MPC-based private inference is gaining wide interest.
%
Existing works on MPC-based inference can be broadly classified into either a \emph{client-server} setup or a \emph{multi-server} setup.

Client-server MPC~\cite{minionn, gazelle, secureml, delphi, cheetah, cryptflow2, ezpc} studies a setup where an MPC-based inference runs between a client holding data and a server holding a model.
In this setup, the server runs most of the heavy computations, assuming that the client device is not powerful (\emph{e.g.}, smartphone or personal laptop)~\cite{reagen_asplos}.
This setup provides strong security where the client does not need to worry about collusion. However, protocols targeting this setup are generally slower because they use a mixture of MPC and homomorphic encryption (HE). These protocols are often called 2PC~\cite{cryptflow2} or hybrid~\cite{aespa} protocols as well.

%
%
%
%

Multi-server MPC~\cite{cryptgpu, aby3, cryptflow, falcon, securenn, charmeleon, astra, blaze, flash, trident, crypten} studies a setup where multiple non-colluding servers collaboratively run an MPC-based inference.
Unlike client-server MPC where one of the parties (the server) does most of the computation, workloads are more balanced in this setup.
While users can also act as one of the parties if they have enough computing power, it is more common to assume they do not participate. Instead, users simply offload the inference to multiple non-colluding servers~\cite{secureml} by generating and sending secret shares of their inputs (Figure~\ref{fig:mpc_intro}).
The servers performing MPC cannot learn about the users' input from the received secret shares unless they collude.
In this setup, the model can be both shared between the parties or be private to one of the parties. If the model is private, participating parties except for the owner of the model use an encrypted model, and the execution is slower compared to when the model is shared.

\begin{figure}
    \centering
    \includegraphics[width=0.49\textwidth]{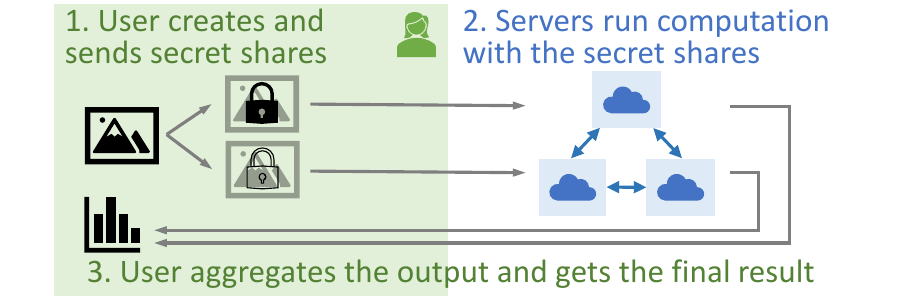}
    \caption{Overview of a multi-server MPC protocol.}
    \label{fig:mpc_intro}
\end{figure}

%
%
%

Multi-server MPC is usually faster than the client-server MPC because it does not involve expensive HE operations --- a recent study~\cite{cheetah} observed a 15$\times$ difference between the two due to the HE operations.
The major downside is that the user data are safe only when the involving parties do not collude~\cite{crypten}.
This non-colluding assumption can be realized with policies and contracts between the parties.
Many companies are forming an alliance~\cite{mpc_alliance} to explore and adopt MPC technologies, and some simple form of MPC is already being adopted in the industry~\cite{meta_mpc}.

\paragraph{Evaluation of ReLU}
For all the MPC protocols, evaluating ReLU consists of a significant portion of the overhead. ReLU is evaluated in several different ways: some of the popular approaches include the Goldreich-Micali-Wigderson (GMW) protocol~\cite{gmw, crypten, aby3, cryptgpu, blaze}, garbled circuit~\cite{yao, aby3, delphi, gazelle, minionn, secureml}, or a variant of SecureNN~\cite{securenn}'s protocol~\cite{securenn, falcon, cryptflow}. Among these, the GMW protocol is GPU-friendly~\cite{cryptgpu} and is often used in GPU-based high-throughput systems~\cite{crypten, cryptgpu}.
%
%

Many of the aforementioned protocols~\cite{gmw, crypten, aby3, cryptgpu, blaze, securenn, falcon, cryptflow} evaluate ReLU by first evaluating whether the secret is positive, \emph{i.e.}, $x \ge 0?$, and multiplying the boolean result by the original secret. Following prior works~\cite{cryptflow}, we call this sign estimation operator \emph{DReLU}~\footnote{for derivative of ReLU}: $\drelu(x)=1$ iff $x \ge 0$ and 0 otherwise.
With DReLU, ReLU is trivially:
\begin{equation}
\label{eq:relu_base}
    \relu(x) = x\times\drelu(x).
\end{equation}
Accelerating the DReLU operation can directly accelerate ReLU for these protocols~\cite{gmw, crypten, aby3, cryptgpu, blaze, securenn, falcon, cryptflow}.

\paragraph{Scope of \sys}
We describe and evaluate the idea of \sys on top of CrypTen~\cite{crypten}, a GMW-based multi-server MPC framework developed and maintained by Meta. CrypTen is popular due to its high-speed GPU support~\cite{crypten} and has served as a foundation of several recent works~\cite{cryptgpu, mpcformer, mpc_transformer}.

While this paper is written around CrypTen, its idea is relevant to a wider range of works --- it is directly applicable to any other protocol that uses Equation~\ref{eq:relu_base} for ReLU and experiences a DReLU overhead that increases with the ring size (\emph{i.e.}, the number of bits in the secret share). All the other GMW-based systems~\cite{aby3, cryptgpu, blaze} and other popular systems~\cite{securenn, falcon, cryptflow, cryptflow2} fall into this category.
%
%
As in the original CrypTen paper, we assume an honest-but-curious adversary~\cite{crypten}.

\subsection{Operation of CrypTen and GMW Protocol}
\label{sec:bg_gmw}

\paragraph{Notations}

Let $x \in \mathbb{Z}/Q\mathbb{Z}$ be a secret value in an integer ring of size $Q=2^N$. We denote $p$ arithmetic secret shares of $x$ as $\arith{x}{Q}_{p} \in \mathbb{Z}/Q\mathbb{Z}$, where $\Sigma_{i=0}^{p-1}\arith{x}{Q}_i \equiv x \Mod{Q}$.
We simply denote the set of the shares as $\arith{x}{Q} = \{\arith{x}{Q}_p\}$.
For $x$ represented in an $N$-bit signed integer representation (two's complement), we denote $p$ binary secret shares of $x$ as $\bin{x}_p$, where $\oplus_{i=0}^{p-1} \bin{x}_i = x$ for a bitwise XOR operation $\oplus$.
Throughout the paper, we assume an element in a ring of size $2^n$ is always in an $n$-bit signed integer representation for any $n$.
We express bits from the $m$-th bit to the $k-1$-th bit in $x$ ($m \le k$) as $x[k:m]$.
For example, if $x = 11011101_b$, $x[5:1]=1110_b$.
Note that the $k$-th bit is excluded.
We treat the resulting $x[k:m]$ as an element on a smaller ring $\mathbb{Z}/2^{k-m}\mathbb{Z}$ unless stated otherwise.
Similarly, we denote the $k$-th bit of $x$ as $x[k]$.

\paragraph{Operation of CrypTen}

%
%
In CrypTen, users split their secret input $x \in \mathbb{Z}/Q\mathbb{Z}$ into $p$ arithmetic secret shares and send each share $\arith{x}{Q}_{p}$ to different participating servers $P_p$.
CrypTen can work with any number of $p \ge 2$;
when $p=2$, secret shares can be easily generated by the client generating a random number $r$ and making $\arith{x}{Q}_0 = x + r$, $\arith{x}{Q}_1 = -r$.
%
%
Floating-point values $x_f$ are converted to an integer ring element $x$ by multiplying with a large integer $D$ and rounding ($x = \lfloor Dx_f \rceil$).

Addition or multiplication by a public value can be trivially done directly on arithmetic secret shares (\emph{e.g.}, $\Sigma_{i=0}^{p-1}a\arith{x}{Q}_i \equiv ax \Mod{Q}$), allowing efficient linear operations (convolution or fully-connected layers) by a public weight.
Addition between two secret shares can also be done trivially without additional overhead.
Multiplication between secret shares adds more overheads because it requires communications between the parties and a set of random numbers called the Beaver triplets~\cite{beaver}. Beaver triplets can be generated and distributed by a trusted third-party (TTP) or using oblivious transfer~\cite{crypten}.
We defer detailed explanations of these arithmetic operations to prior works~\cite{crypten}, as it is not the focus of our optimization.

\paragraph{Evaluating ReLU with GMW}
Non-linear operations, such as max pooling or ReLU, are much more complicated and expensive in MPC.
Here, we describe in detail how ReLU operation is evaluated with the Goldreich-Micali-Wigderson (GMW) protocol, which accounts for more than 93\% of the total execution time (Figure~\ref{fig:teaser}) and is the focus of our paper.

CrypTen evaluates ReLU by separately evaluating DReLU (Equation~\ref{eq:relu_base}).
When DReLU is applied to a secret share $\arith{x}{Q}$, the output is a secret share of one ($\arith{1}{Q}$) if $x \ge 0$ and $\arith{0}{Q}$ otherwise.
ReLU is evaluated by:
\begin{equation}
\label{eq:relu}
    \relu(\arith{x}{Q}) = \arith{x}{Q}\times\drelu(\arith{x}{Q}).
\end{equation}
This requires a multiplication between secret shares and uses the aforementioned Beaver triplets.

Most of the overheads of ReLU come from estimating $\drelu(\arith{x}{Q})$. Below, we explain how the GMW protocol evaluates DReLU. First, the arithmetic secret shares $\arith{x}{Q}$ are converted into binary secret shares $\bin{x}$.
The arithmetic-to-binary (A2B) conversion is done by each party $P_p$ first generating binary secret shares of their arithmetic secret shares, $\bin{\arith{x}{Q}_p}$, and adding their binary shares $\bin{\arith{x}{Q}}_p$ locally~\cite{crypten}. As only bitwise operations like AND or XOR can be done on the binary shares, the addition of $\bin{\arith{x}{Q}}_p$ is performed using a series of AND and XOR operations, as it would be done by an adder circuit (\emph{e.g.}, carry-lookahead adder)~\cite{crypten}.
After the conversion, the most significant bit (MSB; sign bit) of $\bin{x}$ (which is $\bin{x}[$N$-1]$ if $Q=2^N$) holds the binary secret share of 0 if $x$ is positive and 1 if negative. Converting $\bin{x}[$N$-1]$ back into arithmetic secret shares (binary-to-arithmetic; B2A) and subtracting it from a public value 1 gives us our desired $\drelu(\arith{x}{Q})$~\cite{crypten}.

During the circuit addition, XOR can be done locally on each party, similarly to how addition can be done privately on arithmetic secret shares. However, AND, like multiplication between arithmetic secret shares, requires Beaver triplets and communications between the parties. For an $N$-bit secret $x$, the circuit adder implementation requires $O(logN)$ rounds of communication and $O(N)$ bits communicated at each round, resulting in $O(NlogN)$ total communication overheads. Usually, $N$ is large (\emph{e.g.}, 64~\cite{crypten}) to avoid arithmetic wrap-around errors~\cite{crypten, secureml}, and the communication overhead becomes the major bottleneck of DReLU~\cite{crypten, aby3, cryptgpu}.

\subsection{Detailed Overhead Characterization}
\label{sec:bg_detail}

To study the bottleneck of GMW-based MPC protocols, we measured the major overheads of running CrypTen on two nodes with an A100 GPU, connected with a 10 Gbps LAN. More details of the setup can be found in Section~\ref{sec:eval}. We ran ResNet18~\cite{resnet} with CIFAR10~\cite{cifar10} dataset with a batch size of 512. We replaced the max pooling layer with average pooling as in prior works~\cite{reagen_asplos, cryptonite} to concentrate on the ReLU overhead.
While CrypTen (and our proposed optimizations) can be applied both to unencrypted and encrypted models~\cite{crypten}, we assumed that the model is unencrypted and shared among parties, which makes the inference more efficient.
%
%

The leftmost bar in Figure~\ref{fig:teaser} shows the measured overhead breakdown. First, we can see that the numbers are already quite efficient --- finishing an inference of 512 samples in only 26.82 seconds (19.1 samples/s) --- thanks to CrypTen's efficient GPU support. However, the overhead is still significant.
Especially, it can be observed that 93\% of the overhead comes from ReLU layers.
%
As we will show in Section~\ref{sec:eval}, \sys reduces the total communication by 2.68--8.76$\times$, resulting in up to 8.64$\times$ end-to-end speedup (Figure~\ref{fig:teaser}).

\begin{figure}
    \centering
    \includegraphics[width=0.4\textwidth]{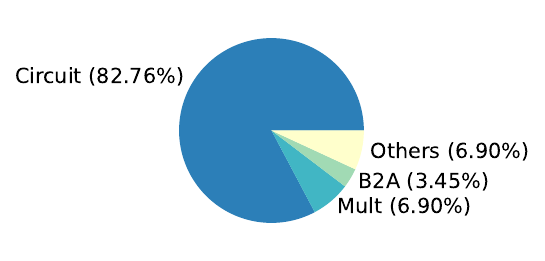}
    \caption{Communication incurred by each part of ReLU.}
    \label{fig:characterize}
\end{figure}

Figure~\ref{fig:characterize} further breaks down the large communication overhead incurred by the ReLU layer into different components.
\textbf{Circuit} refers to the circuit adder explained in Section~\ref{sec:bg_gmw} during the A2B conversion (82.76\%). Specifically, the AND operation inside the circuit adder incurs communication.
\textbf{Mult} refers to the multiplication shown in Equation~\ref{eq:relu} that is done at the end between the secret share and the DReLU output (6.9\%).
\textbf{B2A} refers to the B2A conversion of the 1-bit DReLU output. Unlike the A2B counterpart that performs $N$-bit to $N$-bit conversion, B2A converts only one bit (indicating the sign) and is much cheaper (3.45\%). \textbf{Others} are AND operations happening inside A2B other than what is captured by \textbf{Circuit} (6.9\%).
Evidently, the vast majority of the communication comes from the circuit adder during A2B conversion.

By reducing the number of bits used in DReLU, \sys directly optimizes \textbf{Circuit}, as its communication overhead is $O(NlogN)$ with $N$ bits (Section~\ref{sec:bg_gmw}).
\sys's optimization additionally improves \textbf{Others}, and \sys's efficient bitpacking library (Section~\ref{sec:sys_lib}) also accelerates \textbf{B2A}. \textbf{Mult} cannot be optimized with \sys.


\section{Approximating DReLU with a Subset of Bits}
\label{sec:proof}

The core idea of our optimization is to only use a small fraction of the bits in the secret shares to evaluate the sign of the secret (\emph{i.e.}, DReLU).
Especially, we will show that discarding a certain number of the most- and least-significant bits still allows for correctly estimating the sign.
In other words, for a properly chosen $k$ and $m$ ($k \ge m$), only using $\arith{x}{Q}[k:m]$ to estimate DReLU still gives the correct sign most of the time.
We leverage the fact and propose to use the following approximate equation instead the exact Equation~\ref{eq:relu}:
\begin{equation}
\label{eq:relu_ours}
    \relu(\arith{x}{Q}) \approx \arith{x}{Q}\times\drelu(\arith{x}{Q}[k:m]).
\end{equation}

Figure~\ref{fig:overview} summarizes the proposed approximation, where our unique components are highlighted in blue.
For the GMW protocol, the approximation significantly improves the DReLU complexity from the original $O(NlogN)$ with $N$ bits into $O((k-m)log(k-m))$, where $k - m << N$.
The approximation will also benefit any other protocols whose DReLU overhead decreases with the number of bits~\cite{securenn, falcon, cryptflow, cryptflow2, aby3, cryptgpu, blaze}.

\begin{figure}
    \centering
    \includegraphics[width=0.49\textwidth]{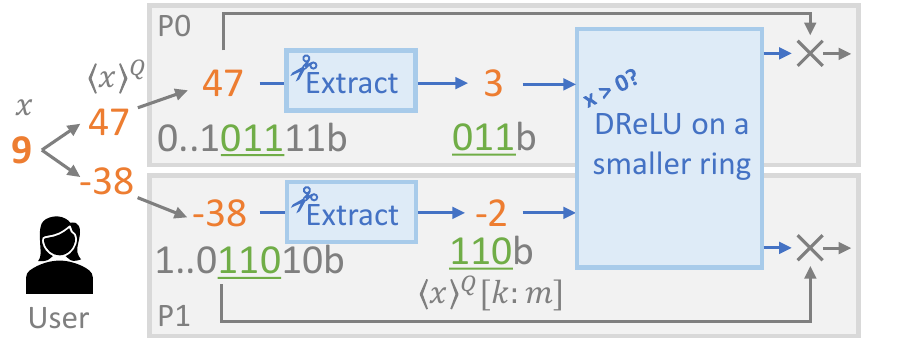}
    \caption{Summary of the proposed approximate ReLU calculation. Our unique contributions are highlighted in blue.}
    \label{fig:overview}
\end{figure}

\subsection{Correctness of the Approximate Algorithm}

In this section, we first explain how the approximate algorithm works in more detail with an example. Then, we theoretically show that the ReLU results stay mostly unchanged if $k$ and $m$ are properly selected; in fact, we will show that the result becomes equivalent to performing a magnitude-based activation pruning after performing exact ReLU.
%

\subsubsection{Example Execution}
We show how the approximate algorithm can still generate a mostly-correct result with an example in Figure~\ref{fig:overview}.
In this example, the user wants to evaluate ReLU on her secret input $x=9$.
The user first generates secret shares $\arith{x}{Q}=\{47,\ -38\}$ and sends each share to different parties, $P_0$ and $P_1$. Note that $47-38=9$ retrieves the original secret value.
Without our optimization, DReLU takes the two secret share values directly as an input and outputs secret shares that indicate the original secret's sign. As the secret ($x=9$) is positive in our example, the output will be $\arith{1}{Q}$. 

In the approximated algorithm, instead of using the shares (47 and -38) directly, each party extracts bits from $k-1$ to $m$ (highlighted in green for $k=5$, $m=2$) and creates new secret shares $\arith{x}{Q}[k:m] = \{3, -2\}$. Note that the bit extraction can be done locally.
The new secret shares can be considered as secret shares of $3-2=1$ in a smaller ring of size $2^3=8$. While the values of the secret shares and the secret value the shares encode all changed significantly (47 $\rightarrow$ 3, -38 $\rightarrow$ -2, 9 $\rightarrow$ 1), note that the sign of the secret value (9 and 1) did not change.
As the secret is still positive, DReLU will still output $\arith{1}{Q}$, and the approximated ReLU result in this example will be \emph{exactly the same} with the precise output.

The reason why the approximation works at a high level is (roughly) because the DReLU result does not change as long as the inequality relationship between the secret shares stays the same. For example, $\arith{x}{Q}= \{47, -38\}$ results in a DReLU output of $\arith{1}{Q}$ because the positive secret share's absolute value (47) is larger than the negative share's (38). This relationship still holds even if we apply, \emph{e.g.}, modulo of 32 (equivalent to dropping high-order bits) or division by 4 (similar to dropping low-order bits) to both shares. In the next section, we provide formal proof of this insight.

\subsubsection{Theoretical Analysis}

In this section, we theoretically prove that the approximate ReLU result is equivalent to magnitude-based activation pruning after performing exact ReLU, with a properly-chosen $k$ and $m$ values.
Our proof is in two steps: we first prove that
(1) removing the $k$-th and higher bits from a secret share does not impact the output of DReLU with a carefully-chosen $k$; then, we prove that (2) removing $m$ low-order bits of a secret share is equivalent to magnitude-based activation pruning.
We only show the proof for a 2-party case (\emph{i.e.}, $p \in \{0, 1\}$) for simplicity; the proof can be extended to more parties trivially.

%
%
%
%

\paragraph{Removing high-order bits}
First, we prove that removing $N-k$ high-order bits of a secret share (\emph{i.e.}, using $\arith{x}{Q}[k:0]$ instead) does not change the DReLU output, if $k$ is selected such that $-2^{k-1} \le x < 2^{k-1}$ holds for all $x$.
The high-level idea of the proof is that $\arith{x}{Q}[k:0]$ can be seen as secret shares of $x[k:0]$ in $\mathbb{Z}/2^{k}\mathbb{Z}$, and hence the DReLU result will be the same if the most significant bit (MSB; sign bit) of $x[k:0]$ is the same as the MSB of $x$.
\begin{theorem}
\label{thm:k}
Consider arithmetic secret shares of $x \in \mathbb{Z}/Q\mathbb{Z}$, $\arith{x}{Q}_p \in \mathbb{Z}/Q\mathbb{Z}$ ($p \in \{0, 1\}$). Assume $\arith{x}{Q}_p$ is represented in an $N$-bit signed integer representation. For $k < N$, $\drelu(\arith{x}{Q}) = \drelu(\arith{x}{Q}[k:0])$ if $-2^{k-1} \le x < 2^{k-1}$.
\end{theorem}
\begin{proof}
$\arith{x}{Q}[k:0]$ can be seen as secret shares of $x[k:0]$ in $\mathbb{Z}/2^{k}\mathbb{Z}$. This is because $\arith{x}{Q}[k:0] \equiv \arith{x}{Q} \Mod{2^{k}}$ and $x[k:0] \equiv x\Mod{2^k}$, and thus, applying $\Mod{2^k}$ to both sides of
\begin{align*}
   \arith{x}{Q}_0 + \arith{x}{Q}_1 \equiv x \Mod{2^{N}}
\end{align*}
results in
\begin{align*}
   \arith{x}{Q}_0[k:0] + \arith{x}{Q}_1[k:0]
   \equiv x[k:0] \Mod{2^{k}}. 
\end{align*}
Applying DReLU to $\arith{x}{Q}[k:0]$ on a smaller ring $\mathbb{Z}/2^{k}\mathbb{Z}$ will simply output secret shares indicating whether its secret ($x[k:0]$) is positive.
Thus, $\drelu(\arith{x}{Q}) = \drelu(\arith{x}{Q}[k:0])$ if and only if their secrets ($x[k:0] \in \mathbb{Z}/2^{k}\mathbb{Z}$ and $x \in \mathbb{Z}/Q\mathbb{Z}$) have the same sign bits, \emph{i.e.}, $x[k-1] = x[N-1]$.
This is always the case if (but not only if) $-2^{k-1} \le x < 2^{k-1}$.
\end{proof}

In MPC frameworks, $N$ is usually chosen to be much larger than what is needed to represent the range of $x$ to avoid wrap-around during arithmetic computation~\cite{secureml, crypten}.
For example, CrypTen~\cite{crypten} uses $N=64$, while a floating point representation $x_f$ is converted into an integer ring element with $x = \lfloor2^{16}x_{f}\rceil$. As intermediate activations ($x_f$) in a DNN are usually close to zero, $x = \lfloor2^{16}x_{f}\rceil$ only occupies a small subset of the full range represented by $N=64$.
For the dataset we studied, $k$ between 18--22 was sufficient for $-2^{k-1} \le x < 2^{k-1}$ to always hold. The result indicates that 42--46 high-order bits (accounting for \textbf{66--72\%}) of the secret shares can be safely discarded without causing \textbf{any} mathematical error.
Unlike linear layers, DReLU does not cause any wrap-around errors and does not need to operate on a large ring.

\paragraph{Removing low-order bits}
Next, we show that discarding $m$ low-order bits in secret shares (\emph{i.e.}, using $\arith{x}{Q}[N:m])$ before DReLU is equivalent to applying magnitude-based activation pruning after ReLU.

\begin{theorem}
\label{thm:m}
Consider arithmetic secret shares of $x$: $\arith{x}{Q}_p \in \mathbb{Z}/Q\mathbb{Z}$ in an $N$-bit signed integer representation. If each party removes $m$ low-order bits of the secret shares and uses $\arith{x}{Q}_p[N:m] \in \mathbb{Z}/2^{N-m}\mathbb{Z}$ for DReLU evaluation, the ReLU output is equivalent to performing ReLU precisely and zeroing-out values below $2^{m}$.
\end{theorem}
\begin{proof}
Note that $\arith{x}{Q}[N:m] = \lfloor\frac{\arith{x}{Q}}{2^{m}}\rfloor$. Consequently,
\begin{align*}
&\arith{x}{Q}_0[N:m] + \arith{x}{Q}_1[N:m]\\
&\equiv \lfloor\frac{\arith{x}{Q}_0}{2^{m}}\rfloor + \lfloor\frac{\arith{x}{Q}_1}{2^{m}}\rfloor\\
&\equiv
\begin{cases}
    \lfloor\frac{x}{2^{m}}\rfloor  \Mod{2^{N-m}}, & \text{or}\\
    \lfloor\frac{x}{2^{m}}\rfloor-1  \Mod{2^{N-m}}.
\end{cases}
\end{align*}

In other words, $\arith{x}{Q}[N:m] \in \mathbb{Z}/2^{N-m}\mathbb{Z}$ are secret shares of either $\lfloor\frac{x}{2^{m}}\rfloor$ or $\lfloor\frac{x}{2^{m}}\rfloor-1$ in $\mathbb{Z}/2^{N-m}\mathbb{Z}$. The sign of the former is always the same as $x$ (here, for simplicity we consider zero as positive, which does not make any difference for ReLU), so applying DReLU yields the same sign as $x$. The latter can cause the sign to flip if (1) $0 < x < 2^{m}$ ($\lfloor\frac{x}{2^{m}}\rfloor$ smaller than 1), or (2) $\lfloor\frac{x}{2^{m}}\rfloor-1 < -2^{N-m-1}$ (underflow).

In CrypTen, $x$'s range is usually much smaller compared to $N$ for the second case to happen. The first case can actually cause an incorrect result, as DReLU will incorrectly consider secrets in $0 < x < 2^{m}$ as negative and output secret shares of zero, which will cause the corresponding ReLU result to become zero. The behavior is equivalent to magnitude-based activation pruning with a threshold $2^{m}$.
\end{proof}
As many prior works~\cite{fatrelu, activation_sparsity, activation_sparsity2, attention_sparsity, attention_sparsity2} empirically showed, magnitude-based activation pruning degrades accuracy gracefully when used in moderation. Thus, a careful choice of $m$ is expected to not harm the model accuracy significantly, having similar effects with prior works on magnitude-based pruning.

\paragraph{Comparison with traditional compression}
What we do is similar in spirit to compression or quantization methods~\cite{deepcompression} in that we aim to reduce the number of bits used. However, traditional compression/quantization aims to make the value of the compressed result close to the original value, \emph{i.e.}, $\compress(\arith{x}{Q}_p) \approx \arith{x}{Q}_p$; however, as $\arith{x}{Q}_p$ are random values fully occupying the $N$-bit representation space (64-bit in CrypTen~\cite{crypten}), it cannot be compressed much without significantly distorting the result. In contrast, our proposed method does not preserve the values of the secret shares at all ($\arith{x}{Q}_p[k:m] \neq \arith{x}{Q}_p$), but it instead ensures that the DReLU result would be similar before and after the bits are discarded ($\drelu(\arith{x}{Q}[k:m]) \approx \drelu(\arith{x}{Q})$).

\paragraph{Applicability to other protocols}
The proof of Theorem~\ref{thm:k} and \ref{thm:m} is not confined to GMW or CrypTen as the proofs do not assume any particular implementation of DReLU.
Thus, the proof is directly applicable to any protocols that calculate DReLU to evaluate ReLU (\emph{i.e.}, use Equation~\ref{eq:relu_base}) and experience DReLU overhead increasing with the ring size. Prior works such as~\cite{aby3, cryptgpu, blaze, securenn, falcon, cryptflow} fall into this category.
\section{\sys System Design}

\begin{figure}
    \centering
    \includegraphics[width=0.49\textwidth]{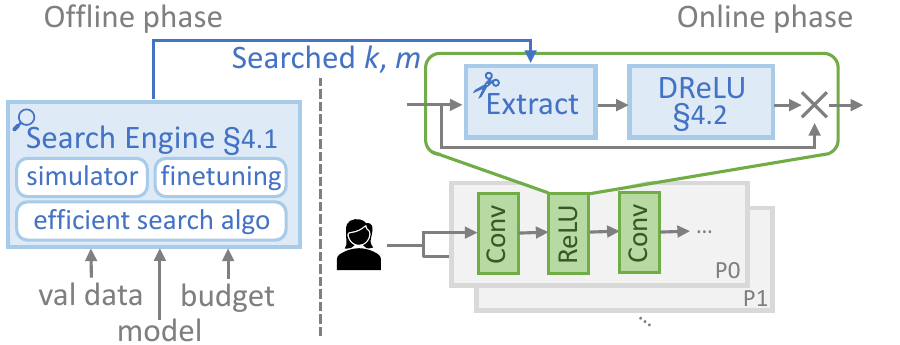}
    \caption{Overview of \sys.}
    \label{fig:sys}
\end{figure}

%
\sys is an MPC framework that allows the user to trade off between performance and accuracy by using the approximate ReLU in Equation~\ref{eq:relu_ours}.
\sys consists of an offline and an online phase. In the \emph{offline phase}, a search engine finds bits to throw out (\emph{i.e.}, select $k$, $m$ in $\arith{x}{Q}[k:m]$) for each ReLU layer to minimize accuracy degradation while maximizing performance. In the \emph{online phase}, \sys uses an efficient runtime library that uses the searched parameters and runs DReLU on a smaller ring to achieve an end-to-end speedup. Figure~\ref{fig:sys} summarizes \sys.

\begin{figure*}
    \centering
    \includegraphics[width=0.98\textwidth]{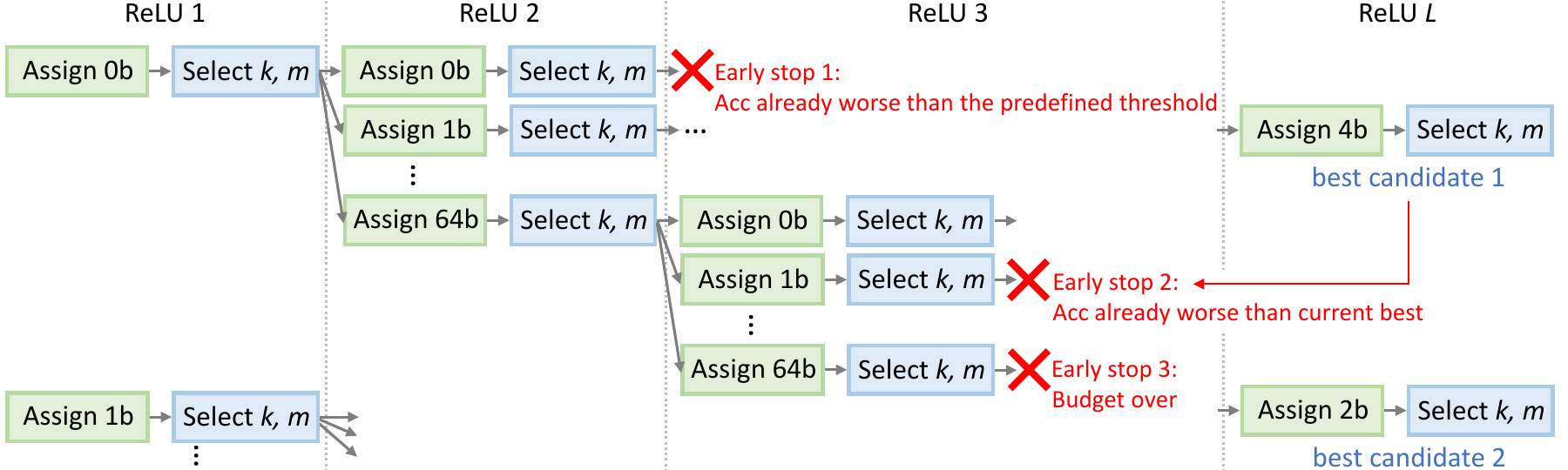}
    \caption{Summary of the search algorithm of \sys. The search enumerates all possible bit assignments in a DFS manner. For each bit assignment, the locally-optimal $k$ and $m$ values are selected. Searching a particular path is immediately stopped when the optimistic accuracy of that path is already worse than the predefined threshold (Early stop 1) or the previous best accuracy found so far (Early stop 2), or when the search budget is exceeded (Early stop 3).}
    \label{fig:dp_search}
\end{figure*}

\subsection{Offline Phase: Finding Bits to Discard}

Each ReLU layer can tolerate different amounts of bits being thrown out. We found that naively throwing out the same bits for every layer leads to suboptimal performance and accuracy (Section~\ref{sec:eval_ablation}). 
To find a tailored number of high- and low-order bits to throw out for each ReLU layer, \sys uses an offline search engine. The search engine consists of three key components: (1) an efficient MPC simulator, (2) a search algorithm, and (3) a model finetuning process at the end to regain some of the lost accuracy.

\subsubsection{MPC Simulator}
\label{sec:sys_sim}
Evaluating any configuration on a real MPC setup during the search process is time-consuming. To save the search time, \sys performs the search on an efficient simulator instead. The simulator simply performs a single-node ML inference (\emph{e.g.}, using popular frameworks like PyTorch) for all layers except ReLU. Only for ReLU layers, the simulator simulates what \sys would do during a real MPC-based inference, \emph{i.e.}, converts the floating point values into an integer ring element, generates secret shares, discards bits, and calculates DReLU using the GMW protocol.

Although the simulator does not simulate what \sys does exactly from the beginning to the end, we observed that the final accuracy trend we observe from the simulator is very similar to what we observe on a real MPC inference.
At the same time, evaluating a configuration on a simulator is much more efficient than running a real MPC, because (1) all the other layers except for ReLU run vanilla single-node inference and incur no additional MPC-related overhead, and (2) even for the ReLU layers, there is no communication overhead as all the parties are simulated on the same node.
The efficient simulator allows the search engine to evaluate numerous configurations within a reasonable time.
%

\subsubsection{Efficient Search Algorithm}
\label{sec:sys_search}
The goal of the search algorithm is to find the subset of bits in the secret shares ($\arith{x}{Q}[k:m]$) to use for the DReLU evaluation for each ReLU layer.
\sys's search engine provides two different search strategies: \sys-eco, which discards as many bits as possible without introducing \emph{any} errors, and \sys-$b$, which discards bits according to the given budget $b$ while minimizing the accuracy degradation.

\paragraph{\sys-eco}
The first approach discards as many bits as possible without introducing any errors. From Theorem~\ref{thm:k} and Theorem~\ref{thm:m}, we have seen that discarding high-order bits to some extent does not alter the ReLU result at all, while discarding low-order bits can always make some non-zero values near zero into zero. Thus, our first approach, \sys-eco, never throws away low-order bits and only throws away high-order bits, such that some performance improvement can be achieved without introducing any errors. In other words, \sys-eco uses $\arith{x}{Q}[k:0]$ with a proper $k$.
Proper $k$ can be selected for each layer by running a validation set while changing $k$ to see if the result changes or not. The search can be done for each layer independently in $O(N)$.

\paragraph{\sys-$b$}
\sys-$b$ takes in the relative amount of bits to retain as the budget $b$ and finds a configuration that meets the budget while maximizing accuracy.
For example, when the search budget is given as 1/16, it means the total number of bits used in each DReLU computation combined must be 1/16 or less than the original number of bits combined.
One way of achieving this is to use only 4 bits among the 64-bit secret shares for all ReLU layers. Alternatively, the budget can be also met by choosing to retain different numbers of bits for different layers (\emph{e.g.}, retain only 2 bits for some layers and 8 bits for other layers). It should be noted that different ReLU layers have different output dimensions --- usually for CNNs, earlier ReLU layers have larger dimensions, and discarding bits from the earlier layers reduces the budget more quickly.

Unlike \sys-eco, the search space for \sys-$b$ grows exponentially with the number of ReLU layers and quickly becomes intractable. With $l$ ReLU layers and $N$ possible bits that can be assigned to each layer, the combinations of possible bit assignments are already $O(N^l)$.
To make matters worse, each ReLU layer has to choose $k$ and $m$ that satisfies the number of assigned bits. For example, if one decides to retain 4 bits for all ReLU layers, each ReLU layer has to choose $k$ and $m$ from $(k, m) \in \{(4, 0), (5, 1),\ ...,\ (64, 60)\}$, resulting in a total $O(N^l)$ possible choices. This leads to a combined $O(N^{2l})$ search complexity.

\sys-$b$ enumerates all possible bit assignments starting from the first ReLU layer in a depth-first-search (DFS) manner (Figure~\ref{fig:dp_search}). To navigate through the exponential search space within a reasonable amount of time, \sys uses several optimizations: using locally-optimal $k$ and $m$ values, early stopping for unlikely paths, and allowing a coarser search.

First, to avoid the $O(N^l)$ complexity of finding a global optimum $k$ and $m$ values, \sys uses a local optimum for each layer instead.
When a certain number of bits is assigned for a layer, the search engine immediately fixes the $k$ and $m$ values for all the other layers and finds the $k$ and $m$ values for the particular layer that gives the best validation accuracy.
This is done by (1) fixing $k$ and $m$ with the already-found values for previous layers that already have been searched, (2) using $k=N$, $m=0$ (\emph{i.e.}, no bit discarded) for successive layers that haven't been searched yet, and (3) linearly enumerating all the possible $k$ and $m$ values that meet the assigned number of bits for the current layer.
The process essentially finds a locally-optimal $k$ and $m$ for each layer while optimistically assuming that successive layers will not degrade the accuracy further.
%
We empirically saw that the heuristic works well.

\begin{figure*}
    \centering
    \includegraphics[width=\textwidth]{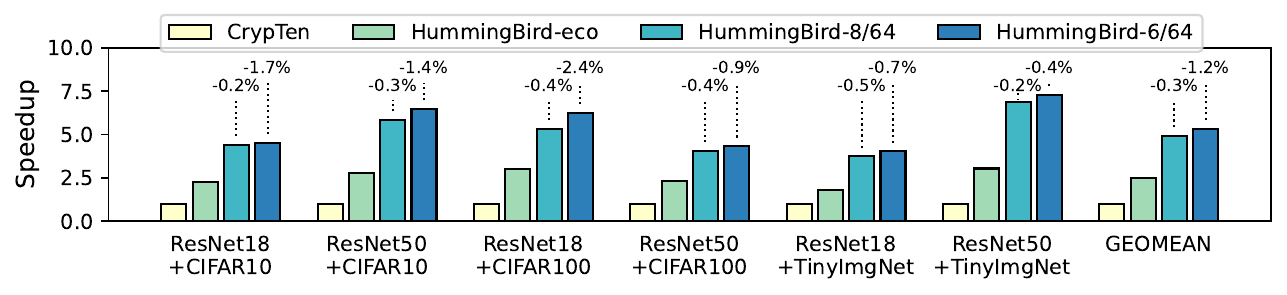}
    \caption{On A100 GPUs, \sys improves the end-to-end performance by 1.81--3.04$\times$ (\sys-eco), 3.74--6.89$\times$ (\sys-8/64), and 4.03--7.28$\times$ (\sys-6/64) over CrypTen. Any accuracy degradation is shown above the bar.}
    \label{fig:perf}
\end{figure*}

Even when we use the locally-optimal $k$ and $m$, navigating all the possible bit assignments with DFS still incurs $O(N^l)$ complexity.
To further make the search tractable, we prune the search space early if a particular branch in the DFS is likely to yield suboptimal configurations.
After assigning a certain number of bits to a layer, the search engine evaluates an optimistic accuracy to find a locally-optimal $k$ and $m$ (discussed in the previous paragraph).
We immediately stop exploring branches whose optimistic accuracy is already worse than a predefined threshold (Figure~\ref{fig:dp_search}, Early stop 1) or the best candidate found so far (Early stop 2).
The insight is that if the optimistic accuracy is already bad, the actual accuracy of any configurations from this branch cannot be good.
We also track the total number of bits assigned to each layer and immediately stop when it exceeds the budget (Early stop 3).

For additional efficiency, we allow performing a search at a larger granularity by grouping multiple ReLUs and making them share the same parameters.
For models with a repeating block structure (\emph{e.g.}, ResNet~\cite{resnet}), a natural choice is to group the ReLUs within the same block.
All these optimizations (using locally-optimal $k$ and $m$, early stopping, and ReLU grouping) combined together allow our search engine to find a good configuration usually within several minutes, making the search highly practical (Section~\ref{sec:eval_search_time}).

When zero bit is assigned to a layer, that ReLU layer becomes an identity layer (\emph{i.e.}, input = output). \sys can be seen as a generalization of ReLU culling~\cite{deepreduce} which replaces a ReLU layer with an identity layer for performance.

\subsubsection{Model Finetuning}
After we find a good configuration, we go through a model finetuning process to regain some of the accuracy drops. The finetuning process is simply done by re-training the model for a small number of epochs with the same training data, while using the approximate ReLU layers with the found parameters. The finetuning process helps the rest of the model to adapt to the approximate ReLU layers. We found that finetuning was not necessary for budgets near 1/8 and above as the approximation does not degrade the accuracy much; however, finetuning was essential for aggressive budgets below 1/8, where non-negligible accuracy drops occurred (Section~\ref{sec:eval_ablation}).


\subsection{Online Phase: Efficient DReLU on a Smaller Ring}
\label{sec:sys_lib}

Using the parameters ($k$ and $m$) found for each ReLU layer, \sys uses the approximate ReLU in Equation~\ref{eq:relu_ours} during online MPC inference.
Note that $k$ and $m$ for each layer are selected during the offline phase using the validation data and are fixed during the online phase, not leaking any additional information about the online user data.

With the reduced number of bits, \sys speeds up the DReLU process, especially the circuit adder (Section~\ref{sec:bg_detail}), with mainly two optimizations.
First, it runs a circuit of depth $O(\lceil log(k - m) \rceil)$ instead of $O(logN)$. Second, it efficiently packs and unpacks the subset of bits into a 64-bit tensor before and after each communication to reduce the overhead. While the circuit depth change only impacts the circuit adder overhead (\textbf{Circuit}; Section~\ref{sec:bg_detail}), the reduced communication due to bitpacking also improves \textbf{Mult} and \textbf{B2A} from Section~\ref{sec:bg_detail}.

We implement \sys's online phase as an extension to the popular CrypTen~\cite{crypten} codebase with Python.
The added code accounts for less than 2\% of the total execution time.

\section{Evaluation Results}
\label{sec:eval}

\begin{table}[]
    \centering
    \caption{Baseline model accuracy.}
    \label{tab:acc}
    \begin{tabular}{ccc}
        \toprule
         Dataset & Model & Accuracy
         \\\midrule
         \multirow{2}{*}{CIFAR10} & ResNet18 & 92.78\%\\
         & ResNet50 & 93.15\%\\\midrule
         \multirow{2}{*}{CIFAR100} & ResNet18 & 77.98\%\\
         & ResNet50 & 79.36\%\\\midrule
         \multirow{2}{*}{\makecell{Tiny-\\ImageNet}} & ResNet18 & 65.46\%\\
         & ResNet50 & 66.87\%\\\bottomrule
    \end{tabular}
\end{table}

In this section, we answer the following questions: 
\begin{itemize}
    \item How faster is \sys in different settings?
    \item How much communication is reduced?
    \item What are the major overheads of \sys?
    \item How long is the search time?
    \item How important are each component of \sys? 
\end{itemize}

\subsection{Evaluation Setup}

\begin{figure*}
    \centering
    \includegraphics[width=\textwidth]{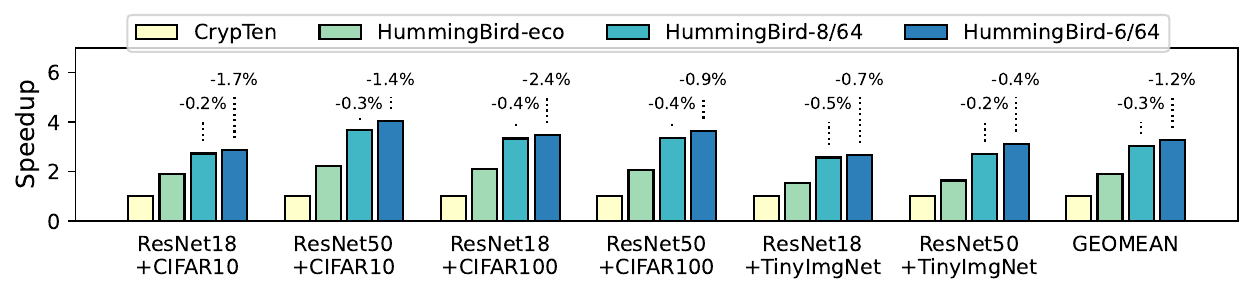}
    \caption{On V100 GPUs, \sys improves the end-to-end performance by 1.55--2.22$\times$ (\sys-eco), 2.57--3.67$\times$ (\sys-8/64), and 2.66--4.03$\times$ (\sys-6/64) over CrypTen. Any accuracy degradation is shown above the bar.}
    \label{fig:perf_aws}
\end{figure*}

\paragraph{System setup}
We evaluate \sys in several representative setups.
The first setup runs two parties on two nodes connected with a 10 Gbps LAN, each with one A100 GPU.
The second setup runs an otherwise identical setup, with a less powerful V100 GPU.
Finally, the third setup runs two parties on two A100 GPUs on a single node. The third represents an ideal setup where the network bandwidth is much higher.
We do not model the overhead of generating Beaver triplets, assuming they are generated and stored offline~\cite{reagen_asplos} or sent by a trusted third-party (TTP) asynchronously.
Unlike in a client-server MPC setup where the clients have limited storage~\cite{reagen_asplos}, we assume the servers have enough storage to hold pre-generated triplets.

\paragraph{Models and datasets}
Following prior works~\cite{reagen_asplos, cryptonite, snl}, we evaluated \sys with ResNet18 and ResNet50~\cite{resnet}, models that are easily supported with MPC with minimal modifications. Popular models like MobileNet~\cite{mobilenet} have components not suitable for MPC (\emph{e.g.}, ReLU6) and are not commonly used. We evaluated three different datasets, CIFAR10~\cite{cifar10}, CIFAR100~\cite{cifar10}, and TinyImageNet~\cite{tinyimagenet}.
For CIFAR10, we replaced the max pooling with average pooling, following~\cite{reagen_asplos, cryptonite}. For the rest, we simply removed max pooling (as average pooling did not work well), following~\cite{snl}.
The baseline accuracy for each model/dataset is summarized in Table~\ref{tab:acc}; the numbers align with prior works~\cite{snl}.
%

\paragraph{\sys parameters}
For the search engine, we used a validation set of 1024 samples and grouped ReLUs into five ReLU groups for faster search, following the five layer groups of ResNet~\cite{resnet}. We used the search budget of 8/64 and 6/64.


\subsection{\sys Performance Analysis}

\paragraph{End-to-end performance improvement}
Figure~\ref{fig:perf} and \ref{fig:perf_aws} show the speedup of \sys over the baseline CrypTen in A100 and V100 GPUs, respectively. The baseline CrypTen uses $N=64$ bits per secret share. \sys-eco discards high-order bits as much as possible without adding errors. \sys-$\{8/64,6/64\}$ uses the parameters found by the search engine when given a budget of 8/64 or 6/64.

Figure~\ref{fig:perf} and \ref{fig:perf_aws} show that \sys improves the end-to-end performance significantly. Without adding any errors (\sys-eco), \sys improves the average performance by 2.49$\times$ and 1.90$\times$ on A100 and V100, respectively.
When some accuracy degradation is tolerated, the average performance improvement becomes 4.93$\times$ and 3.04$\times$ (-0.3\%; \sys-8/64), and 5.34$\times$ and 3.26$\times$ (-1.2\%; \sys-6/64), for A100 and V100, respectively.

The performance improvement is less on the less powerful V100 GPUs because the linear layer computation (\emph{e.g.}, convolution), which \sys does not accelerate, is slower on V100.
The performance improvement discrepancy becomes larger with a tighter, as communication is no longer the sole bottleneck, and computation time becomes more important.

\paragraph{Performance improvement on different networks}
Figure~\ref{fig:analytical} shows the average speedup across all the models/benchmarks for different network setups.
High-BW represents an ideal setup with very high bandwidth. It is measured on two GPUs on a single node, connected with up to 16 Tbps link~\cite{nvlink}.
LAN reports a setup where two nodes each with a GPU are connected with a 10 Gbps LAN. WAN reports an analytical projection assuming a 352 Mbps bandwidth, a WAN bandwidth number used in prior work~\cite{cheetah}. To analyze the end-to-end performance in the WAN setup, we separately measured the communication time from the High-BW setup and scaled it according to the assumed bandwidth. 

\begin{figure}
    \centering
    \includegraphics[width=0.49\textwidth]{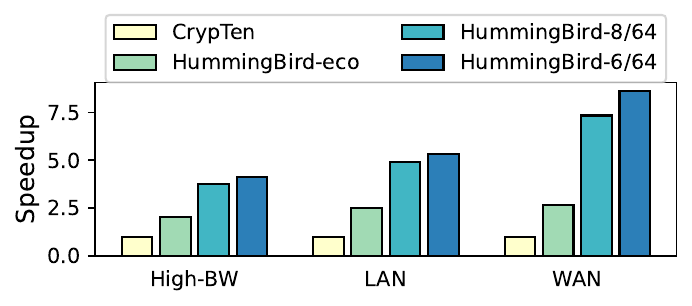}
    \caption{Speedup of \sys on different network setups. The bar shows the geometric mean across all the benchmarks. On WAN, \sys's speedup reaches 2.67--8.64$\times$.}
    \label{fig:analytical}
\end{figure}

\begin{figure}
    \centering
    \includegraphics[width=0.49\textwidth]{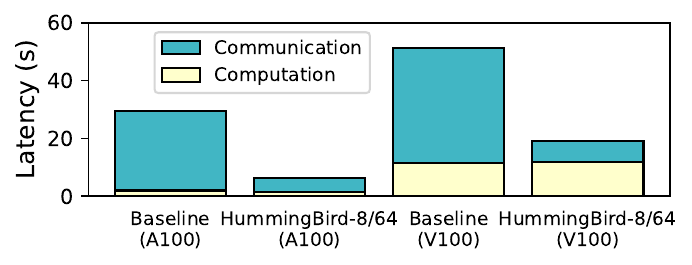}
    \caption{Overhead breakdown of the baseline CrypTen and Hummingbird-8/64 on A100 and V100 GPUs. \sys reduces the communication overhead to a degree where the computation overhead is no longer negligible.}
    \label{fig:breakdown}
\end{figure}

\begin{figure*}
    \centering
    \includegraphics[width=\textwidth]{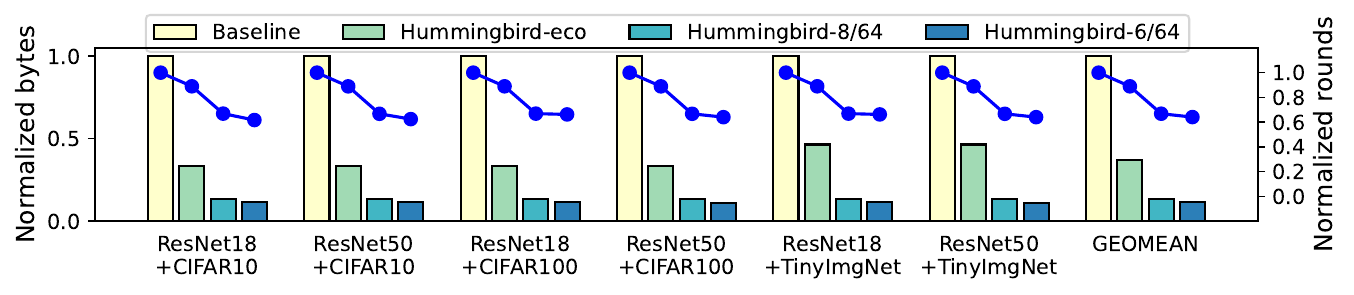}
    \caption{Normalized bytes that need to be communicated (bar) and the number of communication rounds (line). \sys reduces the number of communication rounds by 1.12--1.56$\times$ and total communicated bytes by 2.68--8.76$\times$.}
    \label{fig:comm}
\end{figure*}

Figure~\ref{fig:analytical} shows that, as expected, \sys's performance benefit becomes more notable as the network becomes more limited. Compared to the 2.49--5.34$\times$ speedup of LAN, High-BW setup enjoyed less speedup of 2.03--4.12$\times$, while WAN setup enjoyed more speedup of 2.67--8.64$\times$. High-BW and the LAN setup did not show significant difference although their bandwidth differed by multiple orders, because \sys was not able to fully utilize the bandwidth of High-BW anyways ---  even when the High-BW setup could support up to 16 Tbps, the usage did not exceed 20 Gbps.

\paragraph{Communication}
Figure~\ref{fig:comm} shows the total bytes communicated (bar plot) and the number of communication rounds (line plot). On average, \sys reduces the number of communication rounds by 1.12--1.56$\times$, and reduces the total bytes communicated by 2.68--8.76$\times$. Communication does not decrease proportionally with the budget and starts to saturate because there are communications that cannot be reduced by \sys (\emph{e.g.}, \textbf{Mult} from Figure~\ref{fig:characterize}).

\paragraph{Overhead breakdown} Figure~\ref{fig:breakdown} shows the overhead breakdown of CrypTen and \sys-8/64, both on A100 and V100 GPUs. The breakdown clearly shows that \sys reduces the communication overhead down to a point where the computation overhead becomes non-negligible again.
With \sys-8/64, the portion of the communication overhead decreased from 93\% to 78\% (A100) and 78\% to 39\% (V100), respectively. For high-performance GPUs like A100, the major bottleneck is still communication (78\%); however, for less-powerful GPUs like V100, \sys shifts the major bottleneck to computation.

The result also clearly shows why \sys's speedup is larger for A100 compared to V100. In V100, the computation overhead, which \sys does not accelerate, becomes the major bottleneck.
With \sys, communication is not the sole bottleneck anymore, and future works would have to optimize both the computation and the communication to gain meaningful performance improvements.


\subsection{\sys Search Overhead}
\label{sec:eval_search_time}

Table~\ref{tab:search_time} summarizes the search time of \sys for different setups. In most cases, \sys was able to find a satisfactory configuration in a few minutes. When the dataset and the model were large (\emph{e.g.}, TinyImageNet with ResNet50), the search time became longer, sometimes reaching an hour.
The search time can be further reduced by using a smaller validation set or using a coarser ReLU group.

\begin{table}[]
\centering
    \caption{\sys's configuration search time.}
    \label{tab:search_time}
    \begin{tabular}{cccc}
        \toprule
         Dataset & Model & \multicolumn{2}{c}{Search budget}\\
         & & 8/64 & 6/64
         \\\midrule
         \multirow{2}{*}{CIFAR10} & ResNet18 & 5m 34s & 4m 28s\\
         & ResNet50 & 6m 10s & 5m 47s \\\midrule
         \multirow{2}{*}{CIFAR100} & ResNet18 & 5m 37s & 4m 19s\\
         & ResNet50 & 18m 32s & 18m 34s\\\midrule
         \multirow{2}{*}{\makecell{Tiny-\\ImageNet}} & ResNet18 & 13m 1s & 11m 34s \\
         & ResNet50 & 42m 3s & 1h 8m \\\bottomrule
    \end{tabular}
\end{table}

\subsection{Ablation Studies}
\label{sec:eval_ablation}

\paragraph{Effectiveness of the search engine}
\sys's search engine finds bits to discard (\emph{i.e.}, $k$, $m$) per each ReLU group. A simple alternative approach would be to use the same $k$ and $m$ for all the ReLU layers. We found that such a naive alternative does not work well, incurring more than an 8\% accuracy drop for the same search budget compared to \sys. Figure~\ref{fig:bit_selected} visualizes the bits that are discarded (gray hatched) or retained (green) among the 64 bits  for the two approaches with a budget of 8/64. Unlike the naive approach that discards the same bits for all the ReLU layers (Figure~\ref{fig:bit_selected}, left), \sys flexibly chooses to discard different amounts of bits in different positions (Figure~\ref{fig:bit_selected}, right), sometimes discarding more bits (G3) and sometimes discarding less (G4). 
As different ReLU layers have different importance and characteristics, the search engine is crucial for achieving high accuracy.

\begin{figure}
    \centering
    \includegraphics[width=0.49\textwidth]{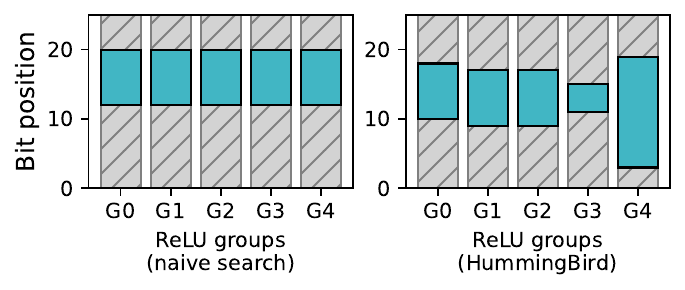}
    \caption{Retained (green) and discarded (grey hatched) bits for each ReLU group on different search strategies. The plot shows that \sys's search engine chooses different s and positions of bits for different ReLU groups.}
    \label{fig:bit_selected}
\end{figure}

\paragraph{Effectiveness of finetuning}
While finetuning was not necessary in cases where the search budget was reasonably large (\emph{e.g.}, \sys-8/64) and the accuracy degradation was already small, we found finetuning to be crucial when the search budget was small (\emph{e.g.}, \sys-6/64) and non-negligible accuracy degradation occurred after discarding bits. Table~\ref{tab:finetune} shows the accuracy before and after finetuning for \sys-6/64.
Finetuning improves the model accuracy by 0.95--7.05\% depending on the dataset and the model.

\begin{table}[]
    \centering
    \caption{Impact of finetuning (FT) on \sys-6/64.}
    \label{tab:finetune}
    \begin{tabular}{cccc}
        \toprule
         Dataset & Model & Before FT& After TF\\\midrule
         \multirow{2}{*}{CIFAR10} & ResNet18 & 90.09\% & 91.04\%\\
         & ResNet50 & 87.6\% & 91.12\%\\\midrule
         \multirow{2}{*}{CIFAR100} & ResNet18 & 73.04\% & 75.57\%\\
         & ResNet50 & 72.45\% & 78.49\%\\\midrule
         \multirow{2}{*}{\makecell{Tiny-\\ImageNet}} & ResNet18 & 60.21\% & 64.79\%\\
         & ResNet50 & 59.82\% & 66.47\%\\\bottomrule
    \end{tabular}
\end{table}

\section{Additional Related Works}

\subsection{Alternative Approaches to Private Inference}

There are multiple alternative approaches to realize private inference. Here, we briefly discuss those alternatives.

\paragraph{Trusted execution environment (TEE)}
TEEs~\cite{sgx, trustzone, amd-sev, keystone, aegis, xom, sp, mi6, tee1, secureme, sanctum, isox, tee2, tee3, tee4, tee5, cheri, tee6} provide hardware-level protection that (1) allows a remote party to authenticate the software that is running on the hardware and (2) ensures the confidentiality and integrity of code and data inside the TEE. Users can send their private data to a remote server's TEE and run inference or training, while ensuring their data stay confidential.
Following the initial proposal from academia~\cite{aegis, xom, sp}, most major vendors have TEEs in their commercial products, including Intel SGX~\cite{sgx}, ARM TrustZone~\cite{trustzone}, AMD SEV~\cite{amd-sev}, RISC-V Keystone~\cite{keystone}, and NVIDIA's recently announced confidential computing feature~\cite{hopper}. Moreover, TEEs for emerging heterogeneous accelerators are also being actively proposed~\cite{hix, cronus, shef, iceclave, tnpu, graviton, tee7, tee8, sealing, guardnn}.
TEEs are efficient because they eliminate the need for any expensive HE or MPC operations, and are widely available in commodity off-the-shelf hardware.
However, the security assurance from a TEE is generally considered to be weaker than cryptographic protection from HE/MPC.
Although data inside a TEE should ideally be secure, 
TEE implementations may be vulnerable due to hardware/software bugs~\cite{tee_bugs, coin_attack} or side channels~\cite{tee_sidechannel}.

\paragraph{Fully homomorphic encryption (FHE)}
FHE is a cryptographic technique that allows certain computations directly on an encrypted ciphertext.
Using FHE, servers can collect user data in a ciphertext form and run computation (\emph{e.g.}, DNN inference) directly on the ciphertext~\cite{fhe}. While the first HE schemes and systems were very slow, subsequent works accelerated HE-based private inference heavily on CPUs~\cite{latigo}, GPUs~\cite{100x, hyphen}, FPGAs~\cite{fpga_he, heax, fab}, and custom accelerators~\cite{f1, bts, ark}. Similar to MPC, non-linear layers such as ReLU incur high overhead and are often approximated with high-degree polynomial functions~\cite{bts}. 
While recent advances in algorithms and hardware accelerators significantly reduced the latency of FHE, the throughput is still limited: using CIFAR10 and ResNet20, recent studies report a throughput of 8 samples/s on a custom accelerator~\cite{ark} and 0.7 samples/s on an A100 GPU~\cite{hyphen}, which are orders of magnitude less than what \sys achieves.

\paragraph{Instance encoding}
Instance encoding~\cite{instahide_broken} refers to a general concept where the client encodes the input into an encoding in a statistical way, such that reconstructing the original input is hard while some useful downstream inference or training is still possible with the encoding.
Similar concepts have been explored under many different names across different communities, including split inference~\cite{neurosurgeon, nopeek-infer, noise1, noise2, shredder, cloak}, split learning~\cite{nopeek, split_learning, split_learning2}, vertical federated learning (vFL)~\cite{sfl, fl_survey, ressfl}, learnable encryption~\cite{neuracrypt, dauntless, imaginary_rotate}, private image publication~\cite{liyue18, liyue19}, \emph{etc.}
Instance encoding is usually efficient computation-wise, as no cryptographically-heavy operation is needed.
However, these approaches lack a strong theoretical guarantee on their claimed privacy-enhancing properties~\cite{instahide_broken, neuracrypt_broken}, unlike MPC or FHE which are shown to be cryptographically secure.
%
A few recent studies provided a theoretical analysis of privacy through instance encoding, using tools like (metric) differential privacy~\cite{liyue18, liyue19}, Fisher information leakage~\cite{maeng_fil, maeng_fil2}, or PAC theory~\cite{pac_privacy}. Still, the theoretical guarantees are much weaker compared to MPC. For example, although instance encoding can make input reconstruction more difficult, it still leaks a non-trivial amount of information about private inputs.

\subsection{Additional Related Works on MPC}

Section~\ref{sec:bg_mpc} summarizes popular client-server and multi-server MPC protocols.
Many of these works simultaneously introduce orthogonal approaches to accelerate ReLU, which are often complementary to ours.
Some of the popular approaches include replacing ReLU with an identity function~\cite{deepreduce, snl}, replacing ReLU with a polynomial~\cite{aespa, delphi, safenet, cryptonet, sisyphus}, and using a neural architecture search to find a model with less number of ReLUs~\cite{deepreduce, deepreshape, sphynx}. As most of these works were not able to fully replace all the ReLUs, \sys will still be beneficial to these systems as well.
Other works focused on applying MPC to more complex models other than CNNs, including Transformers~\cite{mpcformer, mpc_transformer} and recommendation models~\cite{max_pir}.
\sys can still be applied to these works when they use ReLU~\cite{mpcformer, max_pir}.
\section{Conclusion}

MPC-based private inference can allow users to run large models hosted on a remote server without worrying about their private data being leaked. However, running inference using MPC is very slow, due to the significant communication overhead it incurs. A majority ($>$ 93\%) of the overhead comes from the ReLU layers.

In this work, we theoretically show that most of the bits in the secret shares can be removed during ReLU evaluation with little to no impact on accuracy for many popular protocols. Leveraging the finding, we propose \sys, an efficient MPC framework that uses a reduced number of bits during the ReLU evaluation. \sys carefully selects the bits to retain for each layer and uses an efficient runtime library, reducing the communication overhead by up to 8.76$\times$ and achieving up to 8.64$\times$ end-to-end speedup over CrypTen.

\bibliography{references}

\begin{thebibliography}{100}

\bibitem{xray}
T.-T. Ho, K.-D. Tran, and Y.~Huang, ``Fedsgdcovid: Federated sgd covid-19
  detection under local differential privacy using chest x-ray images and
  symptom information,'' {\em Sensors}, vol.~22, no.~10, p.~3728, 2022.

\bibitem{copilot}
{GitHub}, ``{Your AI pair programmer}.''
  \url{https://github.com/features/copilot}, 2023.

\bibitem{alexa}
Amazon, ``Amazon echo \& alexa devices.''
  \url{https://www.amazon.com/smart-home-devices/b?ie=UTF8&node=9818047011},
  2023.

\bibitem{googlehome}
G.~Home, ``Brands you love, united with google home..''
  \url{https://home.google.com/explore-devices/}, 2023.

\bibitem{fbportal}
Meta, ``Meta portal go.''
  \url{https://www.meta.com/portal/products/portal-go/}, 2023.

\bibitem{mpc}
O.~Goldreich, ``Secure multi-party computation,'' {\em Manuscript. Preliminary
  version}, vol.~78, no.~110, 1998.

\bibitem{minionn}
J.~Liu, M.~Juuti, Y.~Lu, and N.~Asokan, ``Oblivious neural network predictions
  via minionn transformations,'' in {\em Proceedings of the 2017 {ACM} {SIGSAC}
  Conference on Computer and Communications Security, {CCS} 2017, Dallas, TX,
  USA, October 30 - November 03, 2017} (B.~Thuraisingham, D.~Evans, T.~Malkin,
  and D.~Xu, eds.), pp.~619--631, {ACM}, 2017.

\bibitem{delphi}
P.~Mishra, R.~Lehmkuhl, A.~Srinivasan, W.~Zheng, and R.~A. Popa, ``Delphi: {A}
  cryptographic inference service for neural networks,'' in {\em 29th {USENIX}
  Security Symposium, {USENIX} Security 2020, August 12-14, 2020} (S.~Capkun
  and F.~Roesner, eds.), pp.~2505--2522, {USENIX} Association, 2020.

\bibitem{gazelle}
C.~Juvekar, V.~Vaikuntanathan, and A.~P. Chandrakasan, ``{GAZELLE:} {A} low
  latency framework for secure neural network inference,'' in {\em 27th
  {USENIX} Security Symposium, {USENIX} Security 2018, Baltimore, MD, USA,
  August 15-17, 2018} (W.~Enck and A.~P. Felt, eds.), pp.~1651--1669, {USENIX}
  Association, 2018.

\bibitem{aby}
D.~Demmler, T.~Schneider, and M.~Zohner, ``{ABY} - {A} framework for efficient
  mixed-protocol secure two-party computation,'' in {\em 22nd Annual Network
  and Distributed System Security Symposium, {NDSS} 2015, San Diego,
  California, USA, February 8-11, 2015}, The Internet Society, 2015.

\bibitem{aby3}
P.~Mohassel and P.~Rindal, ``Aby\({}^{\mbox{3}}\): {A} mixed protocol framework
  for machine learning,'' in {\em Proceedings of the 2018 {ACM} {SIGSAC}
  Conference on Computer and Communications Security, {CCS} 2018, Toronto, ON,
  Canada, October 15-19, 2018} (D.~Lie, M.~Mannan, M.~Backes, and X.~Wang,
  eds.), pp.~35--52, {ACM}, 2018.

\bibitem{crypten}
B.~Knott, S.~Venkataraman, A.~Y. Hannun, S.~Sengupta, M.~Ibrahim, and
  L.~van~der Maaten, ``Crypten: Secure multi-party computation meets machine
  learning,'' in {\em Advances in Neural Information Processing Systems 34:
  Annual Conference on Neural Information Processing Systems 2021, NeurIPS
  2021, December 6-14, 2021, virtual} (M.~Ranzato, A.~Beygelzimer, Y.~N.
  Dauphin, P.~Liang, and J.~W. Vaughan, eds.), pp.~4961--4973, 2021.

\bibitem{cryptflow}
N.~Kumar, M.~Rathee, N.~Chandran, D.~Gupta, A.~Rastogi, and R.~Sharma,
  ``Cryptflow: Secure tensorflow inference,'' in {\em 2020 {IEEE} Symposium on
  Security and Privacy, {SP} 2020, San Francisco, CA, USA, May 18-21, 2020},
  pp.~336--353, {IEEE}, 2020.

\bibitem{cryptflow2}
D.~Rathee, M.~Rathee, N.~Kumar, N.~Chandran, D.~Gupta, A.~Rastogi, and
  R.~Sharma, ``Cryptflow2: Practical 2-party secure inference,'' in {\em {CCS}
  '20: 2020 {ACM} {SIGSAC} Conference on Computer and Communications Security,
  Virtual Event, USA, November 9-13, 2020} (J.~Ligatti, X.~Ou, J.~Katz, and
  G.~Vigna, eds.), pp.~325--342, {ACM}, 2020.

\bibitem{cheetah}
Z.~Huang, W.~Lu, C.~Hong, and J.~Ding, ``Cheetah: Lean and fast secure
  two-party deep neural network inference,'' in {\em 31st {USENIX} Security
  Symposium, {USENIX} Security 2022, Boston, MA, USA, August 10-12, 2022}
  (K.~R.~B. Butler and K.~Thomas, eds.), pp.~809--826, {USENIX} Association,
  2022.

\bibitem{deepreduce}
N.~K. Jha, Z.~Ghodsi, S.~Garg, and B.~Reagen, ``Deepreduce: Relu reduction for
  fast private inference,'' in {\em Proceedings of the 38th International
  Conference on Machine Learning, {ICML} 2021, 18-24 July 2021, Virtual Event}
  (M.~Meila and T.~Zhang, eds.), vol.~139 of {\em Proceedings of Machine
  Learning Research}, pp.~4839--4849, {PMLR}, 2021.

\bibitem{snl}
M.~Cho, A.~Joshi, B.~Reagen, S.~Garg, and C.~Hegde, ``Selective network
  linearization for efficient private inference,'' in {\em International
  Conference on Machine Learning, {ICML} 2022, 17-23 July 2022, Baltimore,
  Maryland, {USA}} (K.~Chaudhuri, S.~Jegelka, L.~Song, C.~Szepesv{\'{a}}ri,
  G.~Niu, and S.~Sabato, eds.), vol.~162 of {\em Proceedings of Machine
  Learning Research}, pp.~3947--3961, {PMLR}, 2022.

\bibitem{securenn}
S.~Wagh, D.~Gupta, and N.~Chandran, ``Securenn: 3-party secure computation for
  neural network training,'' {\em Proc. Priv. Enhancing Technol.}, vol.~2019,
  no.~3, pp.~26--49, 2019.

\bibitem{falcon}
S.~Wagh, S.~Tople, F.~Benhamouda, E.~Kushilevitz, P.~Mittal, and T.~Rabin,
  ``Falcon: Honest-majority maliciously secure framework for private deep
  learning,'' {\em Proc. Priv. Enhancing Technol.}, vol.~2021, no.~1,
  pp.~188--208, 2021.

\bibitem{sphynx}
M.~Cho, Z.~Ghodsi, B.~Reagen, S.~Garg, and C.~Hegde, ``Sphynx: {A} deep neural
  network design for private inference,'' {\em {IEEE} Secur. Priv.}, vol.~20,
  no.~5, pp.~22--34, 2022.

\bibitem{deepreshape}
N.~K. Jha and B.~Reagen, ``Deepreshape: Redesigning neural networks for
  efficient private inference,'' {\em CoRR}, vol.~abs/2304.10593, 2023.

\bibitem{senet}
S.~Kundu, S.~Lu, Y.~Zhang, J.~T. Liu, and P.~A. Beerel, ``Learning to linearize
  deep neural networks for secure and efficient private inference,'' in {\em
  The Eleventh International Conference on Learning Representations, {ICLR}
  2023, Kigali, Rwanda, May 1-5, 2023}, OpenReview.net, 2023.

\bibitem{fatrelu}
M.~Kurtz, J.~Kopinsky, R.~Gelashvili, A.~Matveev, J.~Carr, M.~Goin, W.~M.
  Leiserson, S.~Moore, N.~Shavit, and D.~Alistarh, ``Inducing and exploiting
  activation sparsity for fast inference on deep neural networks,'' in {\em
  Proceedings of the 37th International Conference on Machine Learning, {ICML}
  2020, 13-18 July 2020, Virtual Event}, vol.~119 of {\em Proceedings of
  Machine Learning Research}, pp.~5533--5543, {PMLR}, 2020.

\bibitem{activation_sparsity}
C.~Oh, J.~So, S.~Kim, and Y.~Yi, ``Exploiting activation sparsity for fast
  {CNN} inference on mobile gpus,'' {\em {ACM} Trans. Embed. Comput. Syst.},
  vol.~20, no.~5s, pp.~77:1--77:25, 2021.

\bibitem{activation_sparsity2}
J.~Haberer and O.~Landsiedel, ``Activation sparsity and dynamic pruning for
  split computing in edge ai,'' in {\em Proceedings of the 3rd International
  Workshop on Distributed Machine Learning}, pp.~30--36, 2022.

\bibitem{attention_sparsity}
Z.~Li, C.~You, S.~Bhojanapalli, D.~Li, A.~S. Rawat, S.~J. Reddi, K.~Ye,
  F.~Chern, F.~X. Yu, R.~Guo, and S.~Kumar, ``The lazy neuron phenomenon: On
  emergence of activation sparsity in transformers,'' in {\em The Eleventh
  International Conference on Learning Representations, {ICLR} 2023, Kigali,
  Rwanda, May 1-5, 2023}, OpenReview.net, 2023.

\bibitem{attention_sparsity2}
A.~Gupta, G.~Dar, S.~Goodman, D.~Ciprut, and J.~Berant, ``Memory-efficient
  transformers via top-k attention,'' in {\em Proceedings of the Second
  Workshop on Simple and Efficient Natural Language Processing, SustaiNLP@EMNLP
  2021, Virtual, November 10, 2021} (N.~S. Moosavi, I.~Gurevych, A.~Fan,
  T.~Wolf, Y.~Hou, A.~Marasovic, and S.~Ravi, eds.), pp.~39--52, Association
  for Computational Linguistics, 2021.

\bibitem{secureml}
P.~Mohassel and Y.~Zhang, ``Secureml: {A} system for scalable
  privacy-preserving machine learning,'' in {\em 2017 {IEEE} Symposium on
  Security and Privacy, {SP} 2017, San Jose, CA, USA, May 22-26, 2017},
  pp.~19--38, {IEEE} Computer Society, 2017.

\bibitem{ezpc}
N.~Chandran, D.~Gupta, A.~Rastogi, R.~Sharma, and S.~Tripathi, ``Ezpc:
  Programmable and efficient secure two-party computation for machine
  learning,'' in {\em {IEEE} European Symposium on Security and Privacy,
  EuroS{\&}P 2019, Stockholm, Sweden, June 17-19, 2019}, pp.~496--511, {IEEE},
  2019.

\bibitem{reagen_asplos}
K.~Garimella, Z.~Ghodsi, N.~K. Jha, S.~Garg, and B.~Reagen, ``Characterizing
  and optimizing end-to-end systems for private inference,'' in {\em
  Proceedings of the 28th {ACM} International Conference on Architectural
  Support for Programming Languages and Operating Systems, Volume 3, {ASPLOS}
  2023, Vancouver, BC, Canada, March 25-29, 2023} (T.~M. Aamodt, N.~D.~E.
  Jerger, and M.~M. Swift, eds.), pp.~89--104, {ACM}, 2023.

\bibitem{aespa}
J.~Park, M.~J. Kim, W.~Jung, and J.~H. Ahn, ``{AESPA:} accuracy preserving
  low-degree polynomial activation for fast private inference,'' {\em CoRR},
  vol.~abs/2201.06699, 2022.

\bibitem{cryptgpu}
S.~Tan, B.~Knott, Y.~Tian, and D.~J. Wu, ``Cryptgpu: Fast privacy-preserving
  machine learning on the {GPU},'' in {\em 42nd {IEEE} Symposium on Security
  and Privacy, {SP} 2021, San Francisco, CA, USA, 24-27 May 2021},
  pp.~1021--1038, {IEEE}, 2021.

\bibitem{charmeleon}
M.~S. Riazi, C.~Weinert, O.~Tkachenko, E.~M. Songhori, T.~Schneider, and
  F.~Koushanfar, ``Chameleon: {A} hybrid secure computation framework for
  machine learning applications,'' in {\em Proceedings of the 2018 on Asia
  Conference on Computer and Communications Security, AsiaCCS 2018, Incheon,
  Republic of Korea, June 04-08, 2018} (J.~Kim, G.~Ahn, S.~Kim, Y.~Kim,
  J.~L{\'{o}}pez, and T.~Kim, eds.), pp.~707--721, {ACM}, 2018.

\bibitem{astra}
H.~Chaudhari, A.~Choudhury, A.~Patra, and A.~Suresh, ``{ASTRA:} high throughput
  3pc over rings with application to secure prediction,'' in {\em Proceedings
  of the 2019 {ACM} {SIGSAC} Conference on Cloud Computing Security Workshop,
  CCSW@CCS 2019, London, UK, November 11, 2019} (R.~Sion and C.~Papamanthou,
  eds.), pp.~81--92, {ACM}, 2019.

\bibitem{blaze}
A.~Patra and A.~Suresh, ``{BLAZE:} blazing fast privacy-preserving machine
  learning,'' in {\em 27th Annual Network and Distributed System Security
  Symposium, {NDSS} 2020, San Diego, California, USA, February 23-26, 2020},
  The Internet Society, 2020.

\bibitem{flash}
M.~Byali, H.~Chaudhari, A.~Patra, and A.~Suresh, ``{FLASH:} fast and robust
  framework for privacy-preserving machine learning,'' {\em Proc. Priv.
  Enhancing Technol.}, vol.~2020, no.~2, pp.~459--480, 2020.

\bibitem{trident}
H.~Chaudhari, R.~Rachuri, and A.~Suresh, ``Trident: Efficient 4pc framework for
  privacy preserving machine learning,'' in {\em 27th Annual Network and
  Distributed System Security Symposium, {NDSS} 2020, San Diego, California,
  USA, February 23-26, 2020}, The Internet Society, 2020.

\bibitem{mpc_alliance}
M.~Alliance, ``Powering secure computation, together.''
  \url{https://www.mpcalliance.org/}, 2023.

\bibitem{meta_mpc}
Meta, ``The value of secure multi-party computation.''
  \url{https://privacytech.fb.com/multi-party-computation/}, 2023.

\bibitem{gmw}
O.~Goldreich, S.~Micali, and A.~Wigderson, ``How to play any mental game or {A}
  completeness theorem for protocols with honest majority,'' in {\em
  Proceedings of the 19th Annual {ACM} Symposium on Theory of Computing, 1987,
  New York, New York, {USA}} (A.~V. Aho, ed.), pp.~218--229, {ACM}, 1987.

\bibitem{yao}
A.~C. Yao, ``Protocols for secure computations (extended abstract),'' in {\em
  23rd Annual Symposium on Foundations of Computer Science, Chicago, Illinois,
  USA, 3-5 November 1982}, pp.~160--164, {IEEE} Computer Society, 1982.

\bibitem{mpcformer}
D.~Li, H.~Wang, R.~Shao, H.~Guo, E.~P. Xing, and H.~Zhang, ``{MPCFORMER:} fast,
  performant and provate transformer inference with {MPC},'' in {\em The
  Eleventh International Conference on Learning Representations, {ICLR} 2023,
  Kigali, Rwanda, May 1-5, 2023}, OpenReview.net, 2023.

\bibitem{mpc_transformer}
Y.~Wang, G.~E. Suh, W.~Xiong, B.~Lefaudeux, B.~Knott, M.~Annavaram, and H.~S.
  Lee, ``Characterization of mpc-based private inference for transformer-based
  models,'' in {\em International {IEEE} Symposium on Performance Analysis of
  Systems and Software, {ISPASS} 2022, Singapore, May 22-24, 2022},
  pp.~187--197, {IEEE}, 2022.

\bibitem{beaver}
D.~Beaver, ``Efficient multiparty protocols using circuit randomization,'' in
  {\em Advances in Cryptology - {CRYPTO} '91, 11th Annual International
  Cryptology Conference, Santa Barbara, California, USA, August 11-15, 1991,
  Proceedings} (J.~Feigenbaum, ed.), vol.~576 of {\em Lecture Notes in Computer
  Science}, pp.~420--432, Springer, 1991.

\bibitem{resnet}
K.~He, X.~Zhang, S.~Ren, and J.~Sun, ``Deep residual learning for image
  recognition,'' in {\em 2016 {IEEE} Conference on Computer Vision and Pattern
  Recognition, {CVPR} 2016, Las Vegas, NV, USA, June 27-30, 2016},
  pp.~770--778, {IEEE} Computer Society, 2016.

\bibitem{cifar10}
A.~Krizhevsky, G.~Hinton, {\em et~al.}, ``Learning multiple layers of features
  from tiny images,'' 2009.

\bibitem{cryptonite}
K.~Garimella, N.~K. Jha, Z.~Ghodsi, S.~Garg, and B.~Reagen, ``Cryptonite:
  Revealing the pitfalls of end-to-end private inference at scale,'' {\em
  CoRR}, vol.~abs/2111.02583, 2021.

\bibitem{deepcompression}
S.~Han, H.~Mao, and W.~J. Dally, ``Deep compression: Compressing deep neural
  network with pruning, trained quantization and huffman coding,'' in {\em 4th
  International Conference on Learning Representations, {ICLR} 2016, San Juan,
  Puerto Rico, May 2-4, 2016, Conference Track Proceedings} (Y.~Bengio and
  Y.~LeCun, eds.), 2016.

\bibitem{mobilenet}
M.~Sandler, A.~G. Howard, M.~Zhu, A.~Zhmoginov, and L.~Chen, ``Mobilenetv2:
  Inverted residuals and linear bottlenecks,'' in {\em 2018 {IEEE} Conference
  on Computer Vision and Pattern Recognition, {CVPR} 2018, Salt Lake City, UT,
  USA, June 18-22, 2018}, pp.~4510--4520, Computer Vision Foundation / {IEEE}
  Computer Society, 2018.

\bibitem{tinyimagenet}
{Stanford}, ``{[TinyImageNet download link]}.''
  \url{http://cs231n.stanford.edu/tiny-imagenet-200.zip}, 2023.

\bibitem{nvlink}
{NVIDIA}, ``{NVLink and NVSwitch}.''
  \url{https://www.nvidia.com/en-us/data-center/nvlink/}, 2023.

\bibitem{sgx}
{Intel}, ``{Intel® Software Guard Extensions}.''
  \url{https://software.intel.com/content/www/us/en/develop/topics/software-guard-extensions.html},
  2021.

\bibitem{trustzone}
{ARM Ltd.}, ``{TrustZone for cortex-m}.''
  \url{https://www.arm.com/why-arm/technologies/trustzone-for-cortex-m}, 2021.

\bibitem{amd-sev}
{AMD}, ``{AMD Secure Encrypted Virtualization (SEV)}.''
  \url{https://www.amd.com/en/developer/sev.html}, 2023.

\bibitem{keystone}
D.~Lee, D.~Kohlbrenner, S.~Shinde, K.~Asanovic, and D.~Song, ``Keystone: an
  open framework for architecting trusted execution environments,'' in {\em
  EuroSys '20: Fifteenth EuroSys Conference 2020, Heraklion, Greece, April
  27-30, 2020} (A.~Bilas, K.~Magoutis, E.~P. Markatos, D.~Kostic, and M.~I.
  Seltzer, eds.), pp.~38:1--38:16, {ACM}, 2020.

\bibitem{aegis}
G.~E. Suh, C.~W. O'Donnell, and S.~Devadas, ``Aegis: A single-chip secure
  processor,'' {\em IEEE Design \& Test of Computers}, vol.~24, no.~6,
  pp.~570--580, 2007.

\bibitem{xom}
D.~Lie, J.~C. Mitchell, C.~A. Thekkath, and M.~Horowitz, ``Specifying and
  verifying hardware for tamper-resistant software,'' in {\em 2003 {IEEE}
  Symposium on Security and Privacy (S{\&}P 2003), 11-14 May 2003, Berkeley,
  CA, {USA}}, p.~166, {IEEE} Computer Society, 2003.

\bibitem{sp}
R.~B. Lee, P.~C.~S. Kwan, J.~P. McGregor, J.~S. Dwoskin, and Z.~Wang,
  ``Architecture for protecting critical secrets in microprocessors,'' in {\em
  32st International Symposium on Computer Architecture {(ISCA} 2005), 4-8 June
  2005, Madison, Wisconsin, {USA}}, pp.~2--13, {IEEE} Computer Society, 2005.

\bibitem{mi6}
T.~Bourgeat, I.~A. Lebedev, A.~Wright, S.~Zhang, Arvind, and S.~Devadas,
  ``{MI6:} secure enclaves in a speculative out-of-order processor,'' in {\em
  Proceedings of the 52nd Annual {IEEE/ACM} International Symposium on
  Microarchitecture, {MICRO} 2019, Columbus, OH, USA, October 12-16, 2019},
  pp.~42--56, {ACM}, 2019.

\bibitem{tee1}
D.~Champagne and R.~B. Lee, ``Scalable architectural support for trusted
  software,'' in {\em 16th International Conference on High-Performance
  Computer Architecture {(HPCA-16} 2010), 9-14 January 2010, Bangalore, India}
  (M.~T. Jacob, C.~R. Das, and P.~Bose, eds.), pp.~1--12, {IEEE} Computer
  Society, 2010.

\bibitem{secureme}
S.~Chhabra, B.~Rogers, Y.~Solihin, and M.~Prvulovic, ``Secureme: a
  hardware-software approach to full system security,'' in {\em Proceedings of
  the 25th International Conference on Supercomputing, 2011, Tucson, AZ, USA,
  May 31 - June 04, 2011} (D.~K. Lowenthal, B.~R. de~Supinski, and S.~A. McKee,
  eds.), pp.~108--119, {ACM}, 2011.

\bibitem{sanctum}
V.~Costan, I.~A. Lebedev, and S.~Devadas, ``Sanctum: Minimal hardware
  extensions for strong software isolation,'' in {\em 25th {USENIX} Security
  Symposium, {USENIX} Security 16, Austin, TX, USA, August 10-12, 2016}
  (T.~Holz and S.~Savage, eds.), pp.~857--874, {USENIX} Association, 2016.

\bibitem{isox}
D.~Evtyushkin, J.~Elwell, M.~Ozsoy, D.~V. Ponomarev, N.~B. Abu{-}Ghazaleh, and
  R.~Riley, ``Iso-x: {A} flexible architecture for hardware-managed isolated
  execution,'' in {\em 47th Annual {IEEE/ACM} International Symposium on
  Microarchitecture, {MICRO} 2014, Cambridge, United Kingdom, December 13-17,
  2014}, pp.~190--202, {IEEE} Computer Society, 2014.

\bibitem{tee2}
C.~W. Fletcher, M.~v. Dijk, and S.~Devadas, ``A secure processor architecture
  for encrypted computation on untrusted programs,'' in {\em Proceedings of the
  seventh ACM workshop on Scalable trusted computing}, pp.~3--8, 2012.

\bibitem{tee3}
F.~McKeen, I.~Alexandrovich, A.~Berenzon, C.~V. Rozas, H.~Shafi, V.~Shanbhogue,
  and U.~R. Savagaonkar, ``Innovative instructions and software model for
  isolated execution,'' in {\em {HASP} 2013, The Second Workshop on Hardware
  and Architectural Support for Security and Privacy, Tel-Aviv, Israel, June
  23-24, 2013} (R.~B. Lee and W.~Shi, eds.), p.~10, {ACM}, 2013.

\bibitem{tee4}
J.~Szefer and R.~B. Lee, ``Architectural support for hypervisor-secure
  virtualization,'' {\em ACM SIGPLAN Notices}, vol.~47, no.~4, pp.~437--450,
  2012.

\bibitem{tee5}
D.~Lie, C.~Thekkath, M.~Mitchell, P.~Lincoln, D.~Boneh, J.~Mitchell, and
  M.~Horowitz, ``Architectural support for copy and tamper resistant
  software,'' {\em Acm Sigplan Notices}, vol.~35, no.~11, pp.~168--177, 2000.

\bibitem{cheri}
R.~N.~M. Watson, J.~Woodruff, P.~G. Neumann, S.~W. Moore, J.~Anderson,
  D.~Chisnall, N.~H. Dave, B.~Davis, K.~Gudka, B.~Laurie, S.~J. Murdoch, R.~M.
  Norton, M.~Roe, S.~D. Son, and M.~Vadera, ``{CHERI:} {A} hybrid
  capability-system architecture for scalable software compartmentalization,''
  in {\em 2015 {IEEE} Symposium on Security and Privacy, {SP} 2015, San Jose,
  CA, USA, May 17-21, 2015}, pp.~20--37, {IEEE} Computer Society, 2015.

\bibitem{tee6}
J.~Yang, Y.~Zhang, and L.~Gao, ``Fast secure processor for inhibiting software
  piracy and tampering,'' in {\em Proceedings of the 36th Annual International
  Symposium on Microarchitecture, San Diego, CA, USA, December 3-5, 2003},
  pp.~351--360, {IEEE} Computer Society, 2003.

\bibitem{hopper}
{NVIDIA}, ``{NVIDIA CONFIDENTIAL COMPUTING}.''
  \url{https://www.nvidia.com/en-us/data-center/solutions/confidential-computing/},
  2022.

\bibitem{hix}
I.~Jang, A.~Tang, T.~Kim, S.~Sethumadhavan, and J.~Huh, ``Heterogeneous
  isolated execution for commodity gpus,'' in {\em Proceedings of the
  Twenty-Fourth International Conference on Architectural Support for
  Programming Languages and Operating Systems, {ASPLOS} 2019, Providence, RI,
  USA, April 13-17, 2019} (I.~Bahar, M.~Herlihy, E.~Witchel, and A.~R. Lebeck,
  eds.), pp.~455--468, {ACM}, 2019.

\bibitem{cronus}
J.~Jiang, J.~Qi, T.~Shen, X.~Chen, S.~Zhao, S.~Wang, L.~Chen, G.~Zhang, X.~Luo,
  and H.~Cui, ``{CRONUS:} fault-isolated, secure and high-performance
  heterogeneous computing for trusted execution environment,'' in {\em 55th
  {IEEE/ACM} International Symposium on Microarchitecture, {MICRO} 2022,
  Chicago, IL, USA, October 1-5, 2022}, pp.~124--143, {IEEE}, 2022.

\bibitem{shef}
M.~Zhao, M.~Gao, and C.~Kozyrakis, ``Shef: shielded enclaves for cloud fpgas,''
  in {\em {ASPLOS} '22: 27th {ACM} International Conference on Architectural
  Support for Programming Languages and Operating Systems, Lausanne,
  Switzerland, 28 February 2022 - 4 March 2022} (B.~Falsafi, M.~Ferdman, S.~Lu,
  and T.~F. Wenisch, eds.), pp.~1070--1085, {ACM}, 2022.

\bibitem{iceclave}
L.~Kang, Y.~Xue, W.~Jia, X.~Wang, J.~Kim, C.~Youn, M.~J. Kang, H.~J. Lim, B.~L.
  Jacob, and J.~Huang, ``Iceclave: {A} trusted execution environment for
  in-storage computing,'' in {\em {MICRO} '21: 54th Annual {IEEE/ACM}
  International Symposium on Microarchitecture, Virtual Event, Greece, October
  18-22, 2021}, pp.~199--211, {ACM}, 2021.

\bibitem{tnpu}
S.~Lee, J.~Kim, S.~Na, J.~Park, and J.~Huh, ``{TNPU:} supporting trusted
  execution with tree-less integrity protection for neural processing unit,''
  in {\em {IEEE} International Symposium on High-Performance Computer
  Architecture, {HPCA} 2022, Seoul, South Korea, April 2-6, 2022},
  pp.~229--243, {IEEE}, 2022.

\bibitem{graviton}
S.~Volos, K.~Vaswani, and R.~Bruno, ``Graviton: Trusted execution environments
  on gpus,'' in {\em 13th {USENIX} Symposium on Operating Systems Design and
  Implementation, {OSDI} 2018, Carlsbad, CA, USA, October 8-10, 2018} (A.~C.
  Arpaci{-}Dusseau and G.~Voelker, eds.), pp.~681--696, {USENIX} Association,
  2018.

\bibitem{tee7}
I.~Jang, A.~Tang, T.~Kim, S.~Sethumadhavan, and J.~Huh, ``Heterogeneous
  isolated execution for commodity gpus,'' in {\em Proceedings of the
  Twenty-Fourth International Conference on Architectural Support for
  Programming Languages and Operating Systems, {ASPLOS} 2019, Providence, RI,
  USA, April 13-17, 2019} (I.~Bahar, M.~Herlihy, E.~Witchel, and A.~R. Lebeck,
  eds.), pp.~455--468, {ACM}, 2019.

\bibitem{tee8}
J.~Zhu, R.~Hou, X.~Wang, W.~Wang, J.~Cao, B.~Zhao, Z.~Wang, Y.~Zhang, J.~Ying,
  L.~Zhang, and D.~Meng, ``Enabling rack-scale confidential computing using
  heterogeneous trusted execution environment,'' in {\em 2020 {IEEE} Symposium
  on Security and Privacy, {SP} 2020, San Francisco, CA, USA, May 18-21, 2020},
  pp.~1450--1465, {IEEE}, 2020.

\bibitem{sealing}
P.~Zuo, Y.~Hua, L.~Liang, X.~Xie, X.~Hu, and Y.~Xie, ``Sealing neural network
  models in secure deep learning accelerators,'' {\em CoRR},
  vol.~abs/2008.03752, 2020.

\bibitem{guardnn}
W.~Hua, M.~Umar, Z.~Zhang, and G.~E. Suh, ``Guardnn: Secure {DNN} accelerator
  for privacy-preserving deep learning,'' {\em CoRR}, vol.~abs/2008.11632,
  2020.

\bibitem{tee_bugs}
W.~Liu, H.~Chen, X.~Wang, Z.~Li, D.~Zhang, W.~Wang, and H.~Tang,
  ``Understanding tee containers, easy to use? hard to trust,'' {\em arXiv
  preprint arXiv:2109.01923}, 2021.

\bibitem{coin_attack}
M.~R. Khandaker, Y.~Cheng, Z.~Wang, and T.~Wei, ``{COIN} attacks: On insecurity
  of enclave untrusted interfaces in {SGX},'' in {\em {ASPLOS} '20:
  Architectural Support for Programming Languages and Operating Systems,
  Lausanne, Switzerland, March 16-20, 2020} (J.~R. Larus, L.~Ceze, and
  K.~Strauss, eds.), pp.~971--985, {ACM}, 2020.

\bibitem{tee_sidechannel}
W.~Wang, G.~Chen, X.~Pan, Y.~Zhang, X.~Wang, V.~Bindschaedler, H.~Tang, and
  C.~A. Gunter, ``Leaky cauldron on the dark land: Understanding memory
  side-channel hazards in {SGX},'' in {\em Proceedings of the 2017 {ACM}
  {SIGSAC} Conference on Computer and Communications Security, {CCS} 2017,
  Dallas, TX, USA, October 30 - November 03, 2017} (B.~Thuraisingham, D.~Evans,
  T.~Malkin, and D.~Xu, eds.), pp.~2421--2434, {ACM}, 2017.

\bibitem{fhe}
C.~Gentry, ``Fully homomorphic encryption using ideal lattices,'' in {\em
  Proceedings of the 41st Annual {ACM} Symposium on Theory of Computing, {STOC}
  2009, Bethesda, MD, USA, May 31 - June 2, 2009} (M.~Mitzenmacher, ed.),
  pp.~169--178, {ACM}, 2009.

\bibitem{latigo}
EPFL-LDS, ``Lattigo v2.3.0..'' \url{https://github.com/ldsec/lattigo}, 2021.

\bibitem{100x}
W.~Jung, S.~Kim, J.~H. Ahn, J.~H. Cheon, and Y.~Lee, ``Over 100x faster
  bootstrapping in fully homomorphic encryption through memory-centric
  optimization with gpus,'' {\em {IACR} Trans. Cryptogr. Hardw. Embed. Syst.},
  vol.~2021, no.~4, pp.~114--148, 2021.

\bibitem{hyphen}
D.~Kim, J.~Park, J.~Kim, S.~Kim, and J.~H. Ahn, ``Hyphen: {A} hybrid packing
  method and optimizations for homomorphic encryption-based neural networks,''
  {\em CoRR}, vol.~abs/2302.02407, 2023.

\bibitem{fpga_he}
S.~S. Roy, F.~Turan, K.~J{\"{a}}rvinen, F.~Vercauteren, and I.~Verbauwhede,
  ``Fpga-based high-performance parallel architecture for homomorphic computing
  on encrypted data,'' in {\em 25th {IEEE} International Symposium on High
  Performance Computer Architecture, {HPCA} 2019, Washington, DC, USA, February
  16-20, 2019}, pp.~387--398, {IEEE}, 2019.

\bibitem{heax}
M.~S. Riazi, K.~Laine, B.~Pelton, and W.~Dai, ``{HEAX:} an architecture for
  computing on encrypted data,'' in {\em {ASPLOS} '20: Architectural Support
  for Programming Languages and Operating Systems, Lausanne, Switzerland, March
  16-20, 2020} (J.~R. Larus, L.~Ceze, and K.~Strauss, eds.), pp.~1295--1309,
  {ACM}, 2020.

\bibitem{fab}
R.~Agrawal, L.~de~Castro, G.~Yang, C.~Juvekar, R.~T. Yazicigil, A.~P.
  Chandrakasan, V.~Vaikuntanathan, and A.~Joshi, ``{FAB:} an fpga-based
  accelerator for bootstrappable fully homomorphic encryption,'' in {\em {IEEE}
  International Symposium on High-Performance Computer Architecture, {HPCA}
  2023, Montreal, QC, Canada, February 25 - March 1, 2023}, pp.~882--895,
  {IEEE}, 2023.

\bibitem{f1}
N.~Samardzic, A.~Feldmann, A.~Krastev, S.~Devadas, R.~G. Dreslinski,
  C.~Peikert, and D.~S{\'{a}}nchez, ``{F1:} {A} fast and programmable
  accelerator for fully homomorphic encryption,'' in {\em {MICRO} '21: 54th
  Annual {IEEE/ACM} International Symposium on Microarchitecture, Virtual
  Event, Greece, October 18-22, 2021}, pp.~238--252, {ACM}, 2021.

\bibitem{bts}
S.~Kim, J.~Kim, M.~J. Kim, W.~Jung, J.~Kim, M.~Rhu, and J.~H. Ahn, ``{BTS:} an
  accelerator for bootstrappable fully homomorphic encryption,'' in {\em {ISCA}
  '22: The 49th Annual International Symposium on Computer Architecture, New
  York, New York, USA, June 18 - 22, 2022} (V.~Salapura, M.~Zahran, F.~Chong,
  and L.~Tang, eds.), pp.~711--725, {ACM}, 2022.

\bibitem{ark}
J.~Kim, G.~Lee, S.~Kim, G.~Sohn, M.~Rhu, J.~Kim, and J.~H. Ahn, ``{ARK:} fully
  homomorphic encryption accelerator with runtime data generation and
  inter-operation key reuse,'' in {\em 55th {IEEE/ACM} International Symposium
  on Microarchitecture, {MICRO} 2022, Chicago, IL, USA, October 1-5, 2022},
  pp.~1237--1254, {IEEE}, 2022.

\bibitem{instahide_broken}
N.~Carlini, S.~Deng, S.~Garg, S.~Jha, S.~Mahloujifar, M.~Mahmoody, S.~Song,
  A.~Thakurta, and F.~Tram{\`{e}}r, ``An attack on instahide: Is private
  learning possible with instance encoding?,'' {\em CoRR}, vol.~abs/2011.05315,
  2020.

\bibitem{neurosurgeon}
Y.~Kang, J.~Hauswald, C.~Gao, A.~Rovinski, T.~Mudge, J.~Mars, and L.~Tang,
  ``Neurosurgeon: Collaborative intelligence between the cloud and mobile
  edge,'' {\em ACM SIGARCH Computer Architecture News}, vol.~45, no.~1,
  pp.~615--629, 2017.

\bibitem{nopeek-infer}
P.~Vepakomma, A.~Singh, E.~Zhang, O.~Gupta, and R.~Raskar, ``Nopeek-infer:
  Preventing face reconstruction attacks in distributed inference after
  on-premise training,'' in {\em 2021 16th IEEE International Conference on
  Automatic Face and Gesture Recognition (FG 2021)}, pp.~1--8, IEEE, 2021.

\bibitem{noise1}
T.~Titcombe, A.~J. Hall, P.~Papadopoulos, and D.~Romanini, ``Practical defences
  against model inversion attacks for split neural networks,'' {\em arXiv
  preprint arXiv:2104.05743}, 2021.

\bibitem{noise2}
Z.~He, T.~Zhang, and R.~B. Lee, ``Attacking and protecting data privacy in
  edge--cloud collaborative inference systems,'' {\em IEEE Internet of Things
  Journal}, vol.~8, no.~12, pp.~9706--9716, 2020.

\bibitem{shredder}
F.~Mireshghallah, M.~Taram, P.~Ramrakhyani, A.~Jalali, D.~Tullsen, and
  H.~Esmaeilzadeh, ``Shredder: Learning noise distributions to protect
  inference privacy,'' in {\em Proceedings of the Twenty-Fifth International
  Conference on Architectural Support for Programming Languages and Operating
  Systems}, pp.~3--18, 2020.

\bibitem{cloak}
F.~Mireshghallah, M.~Taram, A.~Jalali, A.~T. Elthakeb, D.~M. Tullsen, and
  H.~Esmaeilzadeh, ``Not all features are equal: Discovering essential features
  for preserving prediction privacy,'' in {\em {WWW} '21: The Web Conference
  2021, Virtual Event / Ljubljana, Slovenia, April 19-23, 2021} (J.~Leskovec,
  M.~Grobelnik, M.~Najork, J.~Tang, and L.~Zia, eds.), pp.~669--680, {ACM} /
  {IW3C2}, 2021.

\bibitem{nopeek}
P.~Vepakomma, A.~Singh, O.~Gupta, and R.~Raskar, ``Nopeek: Information leakage
  reduction to share activations in distributed deep learning,'' in {\em 2020
  International Conference on Data Mining Workshops (ICDMW)}, pp.~933--942,
  IEEE, 2020.

\bibitem{split_learning}
P.~Vepakomma, O.~Gupta, T.~Swedish, and R.~Raskar, ``Split learning for health:
  Distributed deep learning without sharing raw patient data,'' {\em arXiv
  preprint arXiv:1812.00564}, 2018.

\bibitem{split_learning2}
M.~G. Poirot, P.~Vepakomma, K.~Chang, J.~Kalpathy-Cramer, R.~Gupta, and
  R.~Raskar, ``Split learning for collaborative deep learning in healthcare,''
  {\em arXiv preprint arXiv:1912.12115}, 2019.

\bibitem{sfl}
C.~Thapa, M.~A.~P. Chamikara, S.~Camtepe, and L.~Sun, ``Splitfed: When
  federated learning meets split learning,'' in {\em Thirty-Sixth {AAAI}
  Conference on Artificial Intelligence, {AAAI} 2022, Thirty-Fourth Conference
  on Innovative Applications of Artificial Intelligence, {IAAI} 2022, The
  Twelveth Symposium on Educational Advances in Artificial Intelligence, {EAAI}
  2022 Virtual Event, February 22 - March 1, 2022}, pp.~8485--8493, {AAAI}
  Press, 2022.

\bibitem{fl_survey}
Q.~Yang, Y.~Liu, T.~Chen, and Y.~Tong, ``Federated machine learning: Concept
  and applications,'' {\em ACM Transactions on Intelligent Systems and
  Technology (TIST)}, vol.~10, no.~2, pp.~1--19, 2019.

\bibitem{ressfl}
J.~Li, A.~S. Rakin, X.~Chen, Z.~He, D.~Fan, and C.~Chakrabarti, ``Ressfl: {A}
  resistance transfer framework for defending model inversion attack in split
  federated learning,'' in {\em {IEEE/CVF} Conference on Computer Vision and
  Pattern Recognition, {CVPR} 2022, New Orleans, LA, USA, June 18-24, 2022},
  pp.~10184--10192, {IEEE}, 2022.

\bibitem{neuracrypt}
A.~Yala, H.~Esfahanizadeh, R.~G.~L. D'Oliveira, K.~R. Duffy, M.~Ghobadi, T.~S.
  Jaakkola, V.~Vaikuntanathan, R.~Barzilay, and M.~M{\'{e}}dard, ``Neuracrypt:
  Hiding private health data via random neural networks for public training,''
  {\em CoRR}, vol.~abs/2106.02484, 2021.

\bibitem{dauntless}
H.~Xiao and S.~Devadas, ``Dauntless: Data augmentation and uniform
  transformation for learning with scalability and security,'' {\em {IACR}
  Cryptol. ePrint Arch.}, p.~201, 2021.

\bibitem{imaginary_rotate}
L.~Xiang, H.~Zhang, H.~Ma, Y.~Zhang, J.~Ren, and Q.~Zhang, ``Interpretable
  complex-valued neural networks for privacy protection,'' in {\em 8th
  International Conference on Learning Representations, {ICLR} 2020, Addis
  Ababa, Ethiopia, April 26-30, 2020}, OpenReview.net, 2020.

\bibitem{liyue18}
L.~Fan, ``Image pixelization with differential privacy,'' in {\em Data and
  Applications Security and Privacy {XXXII} - 32nd Annual {IFIP} {WG} 11.3
  Conference, DBSec 2018, Bergamo, Italy, July 16-18, 2018, Proceedings}
  (F.~Kerschbaum and S.~Paraboschi, eds.), vol.~10980 of {\em Lecture Notes in
  Computer Science}, pp.~148--162, Springer, 2018.

\bibitem{liyue19}
L.~Fan, ``Differential privacy for image publication,'' in {\em Theory and
  Practice of Differential Privacy (TPDP) Workshop}, vol.~1, p.~6, 2019.

\bibitem{neuracrypt_broken}
N.~Carlini, S.~Garg, S.~Jha, S.~Mahloujifar, M.~Mahmoody, and F.~Tram{\`{e}}r,
  ``Neuracrypt is not private,'' {\em CoRR}, vol.~abs/2108.07256, 2021.

\bibitem{maeng_fil}
K.~Maeng, C.~Guo, S.~Kariyappa, and E.~Suh, ``Measuring and controlling split
  layer privacy leakage using fisher information,'' {\em arXiv preprint
  arXiv:2209.10119}, 2022.

\bibitem{maeng_fil2}
K.~Maeng, C.~Guo, S.~Kariyappa, and G.~E. Suh, ``Bounding the invertibility of
  privacy-preserving instance encoding using fisher information,'' {\em arXiv
  preprint arXiv:2305.04146}, 2023.

\bibitem{pac_privacy}
H.~Xiao and S.~Devadas, ``Pac security: Automatic privacy measurement and
  control of data processing,'' {\em arXiv preprint arXiv:2210.03458}, 2022.

\bibitem{safenet}
Q.~Lou, Y.~Shen, H.~Jin, and L.~Jiang, ``Safenet: {A} secure, accurate and fast
  neural network inference,'' in {\em 9th International Conference on Learning
  Representations, {ICLR} 2021, Virtual Event, Austria, May 3-7, 2021},
  OpenReview.net, 2021.

\bibitem{cryptonet}
R.~Gilad{-}Bachrach, N.~Dowlin, K.~Laine, K.~E. Lauter, M.~Naehrig, and
  J.~Wernsing, ``Cryptonets: Applying neural networks to encrypted data with
  high throughput and accuracy,'' in {\em Proceedings of the 33nd International
  Conference on Machine Learning, {ICML} 2016, New York City, NY, USA, June
  19-24, 2016} (M.~Balcan and K.~Q. Weinberger, eds.), vol.~48 of {\em {JMLR}
  Workshop and Conference Proceedings}, pp.~201--210, JMLR.org, 2016.

\bibitem{sisyphus}
K.~Garimella, N.~K. Jha, and B.~Reagen, ``Sisyphus: {A} cautionary tale of
  using low-degree polynomial activations in privacy-preserving deep
  learning,'' {\em CoRR}, vol.~abs/2107.12342, 2021.

\bibitem{max_pir}
M.~Lam, J.~Johnson, W.~Xiong, K.~Maeng, U.~Gupta, M.~Rhu, H.~S. Lee, V.~J.
  Reddi, G.~Wei, D.~Brooks, and G.~E. Suh, ``Gpu-based private information
  retrieval for on-device machine learning inference,'' {\em CoRR},
  vol.~abs/2301.10904, 2023.

\end{thebibliography}

\end{document}